\PassOptionsToPackage{numbers,sort&compress}{natbib}
\documentclass{article}
\usepackage{microtype}
\usepackage{graphicx}
\usepackage{subfigure}
\usepackage{booktabs} 

\usepackage{hyperref}
\usepackage{url}
\usepackage{algorithm}
\usepackage{algpseudocode}
\usepackage{amsfonts}
\usepackage{enumitem}
\usepackage{appendix}
\usepackage{booktabs}
\usepackage{multirow}
\usepackage{multicol}
\usepackage{amsmath}
\usepackage{amssymb}
\usepackage{mathtools}
\usepackage{amsthm}
\usepackage{float}
\usepackage[main,final]{neurips_2026}
\usepackage[utf8]{inputenc} 
\usepackage[T1]{fontenc}    
\usepackage{hyperref}       
\usepackage{url}            
\usepackage{booktabs}       
\usepackage{amsfonts}       
\usepackage{nicefrac}       
\usepackage{microtype}      
\usepackage{xcolor}         
\usepackage[numbers,sort&compress]{natbib}

\usepackage{hyperref}

\usepackage[capitalize,noabbrev]{cleveref}

\theoremstyle{plain}
\newtheorem{theorem}{Theorem}[section]
\newtheorem{proposition}[theorem]{Proposition}

\newtheorem{corollary}[theorem]{Corollary}
\theoremstyle{definition}
\newtheorem{definition}[theorem]{Definition}
\newtheorem{assumption}[theorem]{Assumption}
\theoremstyle{remark}

\usepackage[textsize=tiny]{todonotes}

\title{CAFE: Causally-Guided Automated Feature Engineering with Multi-Agent Reinforcement Learning}

\author{%
  Arun Vignesh Malarkkan\thanks{Corresponding author: \texttt{arun.malarkkan@asu.edu}} \\
  Arizona State University, AZ, USA \\
  \texttt{arun.malarkkan@asu.edu} \\
  \And
  Wangyang Ying \\
  Arizona State University, AZ, USA \\
  \texttt{wangyang.ying@asu.edu} \\
  \AND
  Hongyu Cao \\
  Clemson University, SC, USA \\
  \texttt{hcao2@clemson.edu} \\
  \And
  Dongjie Wang \\
  Northeast Normal University, Changchun, China \\
  \texttt{wangdongjie100@nenu.edu.cn} \\
  \And
  Yanjie Fu \\
  Arizona State University, AZ, USA \\
  \texttt{yanjie.fu@asu.edu}
}

\begin{document}

\maketitle

\vspace{-0.25cm}
\begin{abstract}
\vspace{-0.25cm}
Automated feature engineering (AFE) enables AI systems to autonomously construct high-utility representations from raw tabular data. However, existing AFE methods rely on statistical heuristics, yielding brittle features that fail under distribution shift.
We introduce \textbf{CAFE}, a framework that reformulates AFE as a causally-guided sequential decision process, bridging causal discovery with reinforcement learning-driven feature construction. 
\emph{Phase I} learns a sparse directed acyclic graph over features and the target to obtain \emph{soft causal priors}, grouping features as direct, indirect, or other based on their causal influence with respect to the target. 
\emph{Phase II} uses a cascading multi-agent deep Q-learning architecture to select causal groups and transformation operators, with hierarchical reward shaping and causal group-level exploration strategies that favor causally plausible transformations while controlling feature complexity.
Across $15$ public benchmarks (classification with macro-F1; regression with inverse relative absolute error), CAFE achieves statistically significant gains over the strongest AFE baseline (Wilcoxon $p<0.001$), reduces episodes-to-convergence, and delivers competitive time-to-target. Under both controlled covariate shifts and naturally occurring distribution shifts, CAFE reduces out-of-distribution performance drop by up to $4\times$ relative to a non-causal multi-agent baseline, and produces more compact feature sets with more stable post-hoc attributions.
These findings underscore that causal structure, used as a soft inductive prior rather than a rigid constraint, can substantially improve the robustness and efficiency of automated feature engineering.
\end{abstract}
\vspace{-0.25cm}
\vspace{-0.2cm}
\section{Introduction}
\vspace{-0.2cm}
In high-stakes domains such as chemical process optimization, industrial manufacturing, drug formulation, and energy systems, AI systems must operate under evolving conditions where input distributions shift. Automated feature engineering (AFE) promises to extract useful representations from raw tabular data with minimal manual effort. Yet, most AFE methods rely on statistical correlations, yielding brittle features that fail once distributions change, limiting use in safety-critical settings.
For instance, in chemical formulation, a 2°C temperature variation can alter product yield by $15$--$30\%$, but correlation-based AFE often attributes outcomes to spurious variables rather than causal drivers. As process conditions evolve, such features lose predictive power despite stable causal mechanisms.
This highlights a fundamental limitation: conventional feature engineering overlooks causality, producing models vulnerable to interventions or operational shifts. We study \textbf{causally-guided AFE}: automatically transforming high-dimensional process data into features that preserve causal relations with target outcomes. This is key for generalization, transparency, and computational tractability.

Yet, causally-guided AFE introduces several core technical challenges: 
1) \textbf{Causal discovery under noise and confounding:} Real-world observational data often includes latent confounders and noise, making reliable causal structure learning difficult without strong assumptions. 
2) \textbf{Combinatorial search explosion:} The space of candidate feature transformations scales exponentially with dimensionality, rendering greedy exploration challenging. 
3) \textbf{Structured exploration under uncertainty:} Navigating the exponential transformation space requires principled search strategies that favor causal plausibility over spurious correlations.
4) \textbf{Robustness:} Features should preserve utility under covariate shifts and mechanism-preserving changes.

The existing methods only partially address these challenges. Most traditional AFE methods rely on greedy exploration and utility-based pruning strategies and can be brittle under shifts~\cite{7344858,horn2019autofeat}.
Reinforcement learning (RL) methods~\cite{chen2019neural, wang2022group, liu2021automated, malarkkan2025delta} improve exploration but suffer from sparse rewards and ignore causal structure.
Recent LLM-based methods~\cite{gong2025evolutionary, zhang2024dynamic, abhyankar2025llm} treat feature generation as sequence modeling but lack principled causal guidance. 
While causal feature selection methods~\cite{tsamardinos2003algorithms, yu2020causality, wang2023data, yao2023causal} identify causally relevant features, they do not synthesize transformed features.
This gap motivates a framework that integrates causality with principled feature generation.

\textbf{Our Insights: A Causally-Guided RL Perspective for Feature Engineering.} AFE must move beyond statistical correlations to leverage causal mechanisms that remain stable across operational changes.
To meet this end, we integrate insights from causal discovery and reinforcement learning: 
First, causal graphs summarize structural dependencies, providing invariant principled guidance for which variables to transform and how to combine them.
Second, RL methods enable adaptive navigation of large, discrete transformation spaces.
We operationalize these foundations through three key insights:
1) \textit{Soft Causal Inductive Bias}: Instead of enforcing rigid causal constraints, the causal structure provides soft inductive priors to guide feature transformation decisions, maintaining flexibility under causal uncertainty;
2) \textit{Causal-aware RL Exploration}: Clustering features by causal roles (direct, indirect, or unrelated to the target), enabling reduced search complexity while preserving causal interpretability;
3) \textit{Causally Shaped Rewards}: Reward functions that balance predictive utility with causal consistency and transformation diversity, ensuring intervention stability.

\textbf{Summary of Proposed Solution:} Building on these insights, we introduce \textbf{CAFE}, a two-phase \underline{C}ausally-guided \underline{A}utomated \underline{F}eature \underline{E}ngineering framework. 
In \textbf{Phase I}, we apply sparsity-regularized causal discovery to construct a causal graph over the feature space, categorizing features as direct causes, indirect causes (via multi-hop paths), or non-causal with respect to the target, providing soft inductive guidance for feature transformation policy while maintaining flexibility under causal uncertainty.
In \textbf{Phase II}, we deploy a cascading multi-agent deep Q-learning architecture where three specialized agents make sequential decisions: selecting causal feature clusters, choosing transformation operators, and constructing causally-informed feature interactions. The agents use hierarchical reward shaping and an adaptive exploration strategy to favor causally plausible transformations while controlling complexity.
Together, these phases unify causal reasoning with reinforcement learning to improve robustness and compactness in high-stakes applications.

\textbf{Our Contributions:} 
1) We formulate AFE as a causally-guided sequential decision process, moving beyond correlation-based heuristics to leverage stable causal mechanisms;
2) We introduce three core principles---soft causal inductive bias, causal structure-aware exploration, and causally shaped reward function, integrating causal discovery with adaptive feature construction through novel multi-agent coordination; 
3) We develop \textbf{CAFE}, a two-phase AFE framework combining causal graph discovery with cascading multi-agent reinforcement learning that strategically constructs causally-informed feature transformations. 
4) We provide extensive empirical validation of CAFE across 15 benchmark datasets, demonstrating consistent gains over correlation-based baselines under a strictly inductive evaluation protocol (significant wins on 14 of 15 datasets), faster convergence with competitive time-to-target, and substantially lower degradation under both controlled and naturally occurring distribution shifts.

\vspace{-0.25cm}
\section{Problem Statement}
Let $\mathcal{D} = \{(\mathbf{x}_i, y_i)\}_{i=1}^N$ be a dataset with $N$ instances, where each $\mathbf{x}_i \in \mathbb{R}^K$ is a $K$-dimensional feature vector and $y_i \in \mathcal{Y}$ is the corresponding target. Let $\mathcal{F}$ denote the original feature set, and $\mathcal{O}$ a predefined set of transformation operators (e.g., arithmetic, interactions; unary or binary with standard domain guards).
Applying transformations from $\mathcal{O}$ to subsets of $\mathcal{F}$ induces a set of candidate features $\mathcal{F}^g$. The goal of AFE is to construct an optimal feature set $\mathcal{F}^* \subseteq \mathcal{F} \cup \mathcal{F}^g$ that maximizes a task-specific performance metric $V_A$ for a given ML task $A$ (e.g., classification, regression).
We formulate the objective as:
\begin{equation}
\label{objective}
\mathcal{F}^* = \arg\max_{\mathcal{S} \subseteq \mathcal{F} \cup \mathcal{F}^g} V_A(\mathcal{S}, y)
\end{equation}
where $\mathcal{S}$ denotes a candidate feature subset and $V_A$ quantifies model performance on task $A$.
The central challenge lies in efficiently navigating the exponentially large transformation space to identify $\mathcal{F}^*$ that achieves: (i) high predictive utility, (ii) robustness under distribution shift, and (iii) computational tractability in exploring the exponential feature space, guided by causality.

\vspace{-0.25cm}

\section{Methodology}
\vspace{-0.2cm}
Automated feature engineering under distribution shift must move beyond correlation-based heuristics to exploit stable causal mechanisms. CAFE addresses this via a two-phase framework that couples causal discovery with multi-agent reinforcement learning (Fig.~\ref{fig:model-architecture}). \emph{Phase I} learns a causal graph to group features by relation to the target, providing soft inductive guidance rather than rigid constraints. \emph{Phase II} uses a factorized multi-agent policy to explore the exponential transformation space, steering by causal priors while remaining flexible under causal uncertainty. Causal guidance enters through probabilistic selection and reward shaping; only safety/validity checks (e.g., guarded division) are enforced as hard constraints. Together, these phases yield compact, transparent feature sets that retain predictive performance under mechanism-preserving covariate shifts. By using the learned causal structure as a soft inductive prior, we transform the AFE problem from a broad, undirected search into a more strategic and efficient exploration, justifying the initial computational overhead of causal discovery with faster convergence to high-quality feature transformations.

\begin{figure*}[t]
  \centering
  \includegraphics[width=0.95\textwidth, height=4.5cm]{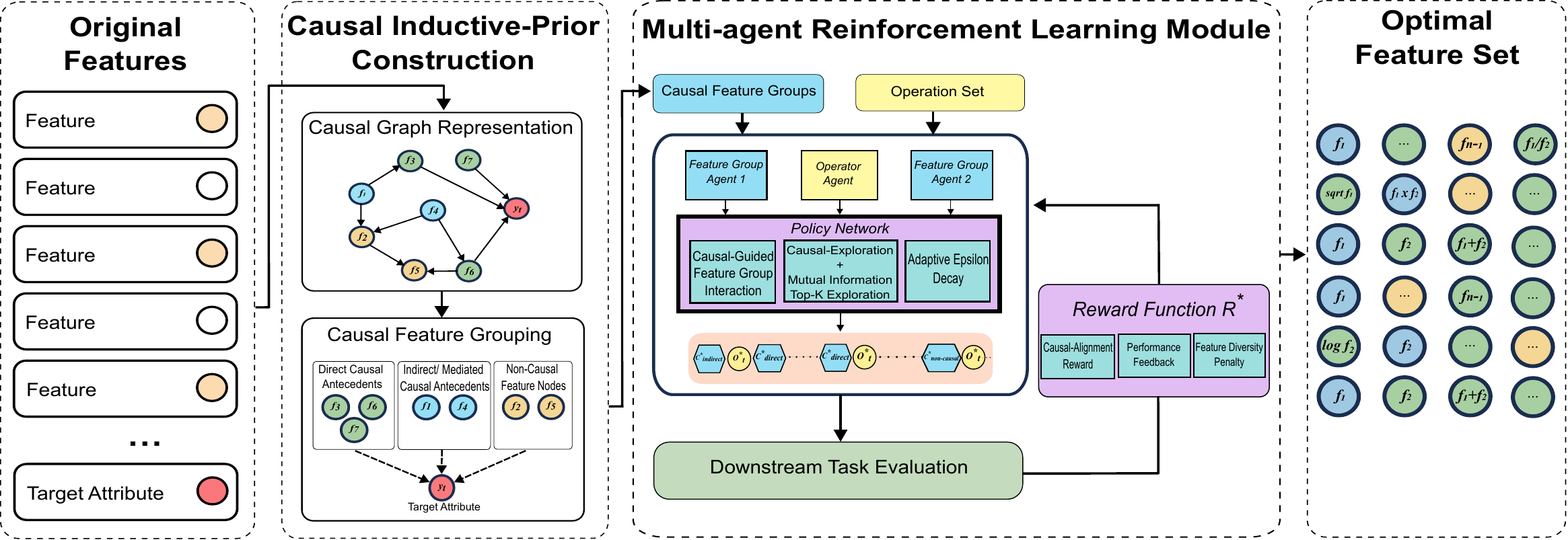}
  \caption{\textbf{CAFE Framework Overview.} Phase I learns a causal graph and derives soft causal priors that group features by their relation to the target: direct causes, indirect causes, and non-causal features. Phase II employs cascaded multi-agent reinforcement learning with causally shaped rewards and adaptive exploration strategies, balancing causal coherence with statistical discovery.}
  \label{fig:model-architecture}
\end{figure*}

\vspace{-0.2cm}
\subsection{Theoretical Foundation: Why Causal Structure Guides Robust Feature Engineering}

\textbf{Core principle.} In a structural causal model (SCM) with assignments $X_j := f_j(\mathrm{PA}_j,\epsilon_j)$, the conditional mechanisms $P(X_j\mid\mathrm{PA}_j)$ remain invariant across environments even when marginals $P(X_j)$ shift~\cite{peters2016causal}. This invariance suggests that features constructed from variables that are causally relevant to $Y$ are more likely to preserve their predictive relation to $Y$ under mechanism-preserving shifts than features relying on spurious correlations.

\textbf{Formal setting.} Consider a structural causal model (SCM) with variables $\mathbf{X} = (X_1, \ldots, X_d)$ and target $Y$, where each variable follows $X_j := f_j(\text{PA}_j, \epsilon_j)$ with parents $\text{PA}_j$ and noise $\epsilon_j$. A mechanism-preserving shift transforms the distribution from $P(\mathbf{X})$ to $P'(\mathbf{X})$ while preserving conditional distributions $P(X_j\mid\text{PA}_j) = P'(X_j\mid\text{PA}_j)$.

\begin{proposition}[Informal invariance for causally sufficient summaries]\label{prop:causal-stability}
Let $S\subseteq\mathcal{F}$ contain a subset $S^\star$ that suffices for predicting $Y$ across environments (e.g., $S^\star \supseteq \mathrm{PA}_Y$ and the shift is mechanism-preserving). If a transformation $\phi$ preserves the information in $S^\star$ (e.g., is injective on $S^\star$ or is a sufficient statistic for $S^\star$), then
\[
P(Y \mid \phi(S)) \;=\; P'(Y \mid \phi(S)).
\]
If $S$ excludes any such invariant set, no general invariance guarantee is available.
\end{proposition}

\noindent\emph{Discussion.} Proposition~\ref{prop:causal-stability} restates a standard invariance intuition~\cite{peters2016causal}: conditioning (directly or via sufficient transforms) on observed causal parents of $Y$ stabilizes the conditional across mechanism-preserving environments. We give precise assumptions and proofs for idealized settings in Appendix~A, and enumerate violations (latent confounding, selection bias, mechanism changes) that break invariance.

\textbf{Empirical regularity.} When causal discovery yields a reasonable approximation to the true relations, features built from variables with stronger (estimated) causal connection to $Y$ empirically exhibit greater stability under mechanism-preserving shifts than features built from purely correlational signals. We quantify this trend via controlled robustness experiments (Sec.~4.2.3) and sensitivity analyses under graph misspecification (Appendix~B).

\noindent This perspective motivates organizing features by approximate causal relationships: direct causes ($X_j \to Y$), indirect causes ($X_j \leadsto Y$ via paths of length $\geq 2$), and other variables. This causal hierarchy acts as a soft inductive bias that guides but does not constrain feature construction decisions.

\vspace{-0.2cm}
\subsection{Phase I: Causal Graph Discovery as a Soft Inductive Prior for Feature Grouping}

In Phase I, we construct a causal graph to organize the feature space by underlying causal relationships rather than statistical correlations. The central insight is that features with direct or indirect causal influence on the target provide more reliable guidance for feature construction, offering stable inductive signals under mechanism-preserving shifts, whereas correlations may be spurious.

\textbf{Causal Discovery Backend.} Given an observational dataset $\mathcal{D} = \{(\mathbf{x}_i, y_i)\}_{i=1}^N$, we construct a Directed Acyclic Graph (DAG) $\mathcal{G} = (\mathcal{V}, \mathcal{E})$ over the variable set $\mathcal{V} = \mathcal{F} \cup \{y\}$, where $\mathcal{F}$ represents the original features. Each directed edge $(u, v) \in \mathcal{E}$ represents a causal relationship from variable $u$ to variable $v$. 

We employ NOTEARS with Lasso regularization~\cite{10.5555/3327546.3327618} as our primary causal discovery method, though the framework accommodates alternative algorithms (Appendix~B). Given dataset $\mathcal{D} = \{(\mathbf{x}_i, y_i)\}_{i=1}^N$ with features $\mathcal{F}$ and target $y$, we learn a DAG $\mathcal{G} = (\mathcal{V}, \mathcal{E})$ over $\mathcal{V} = \mathcal{F} \cup \{y\}$ by optimizing:
\vspace{-0.15cm}
\begin{equation}
\begin{aligned}
\min_{W \in \mathbb{R}^{d \times d}} \quad
& \frac{1}{2N}\lVert X - XW\rVert_F^2 + \lambda \lVert W\rVert_1 \\
\text{s.t.} \quad
& h(W) = \operatorname{tr}\!\left(e^{W \circ W}\right) - d = 0 .
\end{aligned}
\label{eq:notears}
\end{equation}

where $W$ is the weighted adjacency matrix, $\lambda > 0$ controls sparsity, $\circ$ denotes Hadamard product, and $h(W) = 0$ enforces acyclicity via a differentiable constraint. This formulation enables gradient-based optimization while maintaining theoretical guarantees about DAG structure. The optimized adjacency matrix $W$ defines the edge weights of the causal graph $\mathcal{G} = (\mathcal{V}, \mathcal{E})$, where larger weights indicate stronger causal influence. This produces a sparse DAG that captures the most significant causal pathways while filtering out weak or spurious connections.

\textbf{Assumptions and Scope.} Our main analysis assumes: (1) causal sufficiency (no unobserved confounders), (2) linear additive noise model with sparse structure, and (3) mechanism-preserving shifts. We evaluate robustness to assumptions violation and compare alternative discovery methods like PC~\cite{kalisch2007estimating}, GES~\cite{chickering2002optimal}, LiNGAM~\cite{shimizu2014lingam} in Appendix~B.

\textbf{Causal Role Assignment.} From learned DAG $\mathcal{G}$, each feature $f \in \mathcal{F}$ receives a causal role:
\vspace{-0.15cm}
\begin{equation}
\mathcal{M}(f) = \begin{cases}
\text{direct} & \text{if } f \to y \\
\text{indirect} & \text{if } f \leadsto y \text{ (path length } \geq 2\text{)} \\
\text{other} & \text{otherwise}
\end{cases}
\end{equation}

This induces causal-semantic groups $\{\mathcal{C}_{\text{direct}}, \mathcal{C}_{\text{indirect}}, \mathcal{C}_{\text{other}}\}$. To balance signal preservation with computational tractability, we apply within-group screening:
\vspace{-0.15cm}
\begin{equation}
\mathcal{C}_g^* = \begin{cases}
\mathcal{C}_g & \text{if } |\mathcal{C}_g| \leq k_g \\
\text{top-}k_g(\mathcal{C}_g, \text{MI}(\cdot, y)) & \text{otherwise}
\end{cases}
\label{eq:group-screening}
\end{equation}

where top-$k_g$ selects features with highest mutual information with target $y$. Robust fallbacks (Pearson correlation, $\chi^2$ test, stratified sampling) handle cases where MI estimation is unreliable.

These causal roles serve as soft inductive priors and guide feature transformation decisions without rigid constraints. This design maintains flexibility when causal discovery is imperfect while providing a principled structure that statistical methods lack.

\subsection{Phase II: Multi-Agent Reinforcement Learning with Causal Guidance}

Phase I establishes causal-semantic feature groups through theoretically grounded structure discovery. Phase II addresses the exponentially large transformation space ($O(|\mathcal{O}| \cdot |\mathcal{F}|^2)$) via multi-agent reinforcement learning with a Decision Factorization Rationale, where agents select entire causal feature groups rather than individual features.

\textbf{Decision Factorization Rationale.} The AFE action space is combinatorially large: $O(|\mathcal{O}| \times |\mathcal{F}|^2)$ for operator set $\mathcal{O}$ and feature set $\mathcal{F}$. Direct optimization suffers from sparse rewards and high variance. We factorize decisions into three stages: (1) select primary feature group, (2) choose transformation operator, (3) optionally select secondary group for binary operations. This factorization substantially reduces the effective action space each agent explores; from joint choices over $(\text{group}, \text{operator}, \text{partner})$ to three smaller, structured decisions, while preserving semantic coherence through causal organization. Empirically, this reduces value-estimate variance and improves sample efficiency (learning-curve dispersion in Appendix~A). We also include a stylized analysis clarifying conditions where factorization reduces estimation variance; we treat it as a modeling advantage rather than a general guarantee.

\textbf{Multi-Agent Architecture.} Three specialized DQN agents operate in cascade:
\emph{Primary Group Agent} ($\pi_1$): Selects from $\{\mathcal{C}_{\text{direct}}^*, \mathcal{C}_{\text{indirect}}^*, \mathcal{C}_{\text{other}}^*\}$.
\emph{Operator Agent} ($\pi_o$): Chooses transformation from operator library $\mathcal{O}$. \emph{Secondary Group Agent} ($\pi_2$): Selects partner group for binary operations

\textbf{State Representation Design.} Agents receive context through statistical feature summaries and group-specific information:
\begin{align}
s_t^{(1)} &= \text{NestedStats}(\mathcal{F}_{t-1}) \oplus \text{GroupStats}(\mathcal{C}_{\text{direct}}^*, \mathcal{C}_{\text{indirect}}^*, \mathcal{C}_{\text{other}}^*) \label{eq:state1}\\
s_t^{(o)} &= s_t^{(1)} \oplus \text{NestedStats}(\mathcal{C}_t^{(1)}) \label{eq:state2}\\
s_t^{(2)} &= s_t^{(o)} \oplus \text{OneHot}(o_t) \label{eq:state3}
\end{align}

where $\mathrm{NestedStats}(X)$ first computes eight descriptive statistics (count, mean, standard deviation, minimum, 25th percentile, median, 75th percentile, and maximum) for each feature and then applies the same eight statistics across features, yielding a 64-dimensional permutation-invariant representation. $\mathrm{GroupStats}$ concatenates the corresponding representations of the direct, indirect, and other causal groups. Consequently, the three agents receive fixed-dimensional state vectors of 256, 320, and 335 dimensions, respectively. This representation captures both within-feature and across-feature distributional structure without introducing an additional learned encoder; full details are provided in Appendix~\ref{app:state_representation}.

\textbf{Causally-Guided Exploration.} We design three principled feature interaction strategies that integrate causal theory, statistical relevance, and diversity discovery with adaptive scheduling:

\textit{\underline{Causal-Hierarchical Feature Interaction}}: Prioritizes transformations respecting causal structure. For unary operations, favors $\mathcal{C}_{\text{direct}} \cup \mathcal{C}_{\text{other}}$ over indirect causes (which mediate relationships). For binary operations, emphasizes $\mathcal{C}_{\text{direct}} \times \mathcal{C}_{\text{indirect}}$ and $\mathcal{C}_{\text{direct}} \times \mathcal{C}_{\text{direct}}$ combinations.
\\
\textit{\underline{Top-$k$ Mutual Information Selection}}: Hedges against causal discovery errors by incorporating mutual information rankings, selecting features with strongest empirical associations with target.
\\
\textit{\underline{Diversity Sampling}}: Prevents premature convergence through controlled randomization, ensuring adequate exploration of the transformation space.

Strategy weights adapt based on validation performance using exponential moving averages with decay $\alpha$ (Appendix).

\paragraph{Reward Function Design.}
We align optimization with predictive utility and causal principles:
\begin{equation}
R_t
=
R_{\mathrm{perf},t}
\bigl(1 + \alpha\,\Psi_{\mathrm{causal},t}\bigr)
\;+\;
\lambda_{\mathrm{div}}\,\mathcal{H}(\pi_t)
\;-\;
\lambda_{\mathrm{comp}}\,C\!\left(\mathcal{F}_t^{g}\right),
\label{eq:reward}
\end{equation}
where $R_{\mathrm{perf},t} = \Delta \mathrm{ValidationScore}_t$ denotes the change in validation performance at step~$t$, 
$\alpha \in (0,1)$ controls causal bonus intensity,
$\mathcal{H}(\pi_t)$ is the entropy of the policy $\pi_t$ encouraging exploration,
and $C(\mathcal{F}_t^{g})$ penalizes model complexity induced by the selected feature set $\mathcal{F}_t^{g}$.

The causal contribution is defined as
\begin{equation}
\Psi_{\mathrm{causal},t}
=
\frac{1}{\lvert \mathcal{F}_t^{g} \rvert}
\sum_{f \in \mathcal{F}_t^{g}}
w_{\mathcal{M}(f)} \,\mathrm{Usage}(f),
\end{equation}
where $\mathcal{M}(f)$ maps each feature to its causal category, and
\[
w_{\mathrm{direct}} > w_{\mathrm{indirect}} > w_{\mathrm{other}} \ge 0.
\]

The modulation term $(1 + \alpha \Psi_{\text{causal},t})$ amplifies performance gains attributable to causally relevant features while attenuating negative updates when such features are present. This design discourages reward hacking by reducing the influence of purely spurious improvements on the learning signal, while preserving a preference for causal structure. The complete algorithm is provided in Appendix~A (Algorithm~\ref{alg:cafe}).

\textbf{Robust Feature Generation.} Transformation operators include domain validation and numerical safeguards addressing common failure modes: 

\emph{\underline{Logarithmic operations}}: Apply to $\log(|x| + \epsilon)$ where $\epsilon = 10^{-8}$ for numerical stability.
\\
\emph{\underline{Division operations}}: Use protected division $x \cdot \frac{\mathrm{sign}(y)}{\max(|y|,\epsilon)}$ with $\epsilon=10^{-8}$.
\\
\emph{\underline{Square root}}: Apply $\sqrt{|x|}$ for negative inputs.
\\
\emph{\underline{Range clipping}}: Bound generated features to $[-10^6, 10^6]$ preventing overflow.

For binary operations, we limit candidate pairs to $\min(50, |\mathcal{C}_1| \times |\mathcal{C}_2|)$ through relevance-based sampling, maintaining computational tractability while preserving quality.

\textbf{Learning and Optimization.} Agents train via temporal-difference learning with experience replay and target networks:
\begin{equation}
\mathcal{L}(\theta) = \mathbb{E}_{(s,a,r,s') \sim \mathcal{B}}\left[\left(Q_\theta(s,a) - \left(r + \gamma \max_{a'} Q_{\theta^-}(s',a')\right)\right)^2\right]
\end{equation}

\vspace{-0.3cm}
where $\mathcal{B}$ is the replay buffer, $\theta^-$ denotes target network parameters, and $\gamma$ is the discount factor.

\textbf{Scalability and Termination.}
We implement intelligent pruning and adaptive termination to ensure computational efficiency while maintaining feature quality. Two-stage filtering first applies variance thresholding to remove near-constant features, then uses mutual information ranking to retain top-k features while ensuring minimum representation per causal group.
Multiple termination criteria trigger when: (1) validation improvement falls below $\delta$ for N consecutive episodes; (2) feature count exceeds computational limit; or (3) maximum episodes are reached. This design choice makes the group-choice action constant-size ($|G|=3$), improving exploration stability. With a fixed operator library and capped within-group sampling ($k_g$) and pairing budget $B=\min(50, |C_1|\cdot|C_2|)$, the per-step decision cost is $O(1)$ for group selection and $O(|\mathcal{O}| + B)$ for operator and pairing, independent of raw dimensionality $d$. Overall runtime remains dominated by candidate evaluation, pruning, and downstream model training.
Upon termination, we return the feature set $\mathcal{F}^*$, achieving the highest validation performance across all episodes.
\paragraph{Implementation summary.}
CAFE is implemented using three DQN agents with a fixed architecture and training protocol. Each agent uses a 3-layer MLP with hidden sizes [512, 256, 128], ReLU activations, and a linear output head. Agents are trained with experience replay (buffer size 50k, batch size 256), target networks updated every 1000 steps, discount factor $\gamma=0.99$, and the Adam optimizer with learning rate $10^{-4}$. Training terminates early when validation improvement falls below $\delta$ for $N$ consecutive episodes or when a feature budget is exceeded. 

The complexity penalty is defined as $C(F_t^g)=0.001|F_t^g| + 0.01\sum_f \mathrm{OpDepth}(f)$, discouraging overly large or deep feature constructions. Complete state encodings, reward components, hyperparameters, and stopping criteria are provided in Appendix C–D for full reproducibility.

\vspace{-0.25cm}
\vspace{-0.1cm}
\section{Experimental Study}

We designed our experiments to meticulously address the following research questions:
\textbf{RQ1: Predictive Utility.} Does CAFE consistently outperform recent AFE baselines in terms of predictive accuracy?
\textbf{RQ2: Ablation Study.} How do causal priors, exploration strategies, and reward components individually contribute to the performance and efficiency?
\textbf{RQ3: Distributional Robustness.} Can CAFE generate features that are resilient under covariate shifts and maintain predictive accuracy on out-of-distribution data?
\textbf{RQ4: Interpretability and Compactness.} Are the synthesized features semantically meaningful and causally aligned with the underlying data-generating process?

\vspace{-0.2cm}
\subsection{Experimental Protocol}
\label{sec:experimental_protocol}

\textbf{Data Description and Evaluation Framework.}
We evaluate our framework on 15 diverse public datasets collected from Kaggle, LibSVM, UCIrvine, and OpenML repositories, across two core tasks: (1) classification and (2) regression. The data statistics are reported in Table~\ref{table_overall_perf}.
We use a strictly inductive nested evaluation protocol. For each outer cross-validation fold, the outer test partition is held out before any target-dependent feature engineering. All Phase~I operations, including NOTEARS causal discovery, causal-role assignment, and within-group mutual-information screening, are fitted exclusively on the corresponding outer-training partition. The outer-training partition is further divided into disjoint search-training and inner-validation subsets. Phase~II uses the search-training subset for feature construction and the inner-validation subset for RL rewards, early stopping, and feature-program selection. The selected feature program is then frozen and evaluated on the untouched outer-test partition. Thus, test samples and labels are never used for causal discovery, feature grouping, RL rewards, feature-program selection, or downstream-model fitting.
All experiments employ 5-fold outer cross-validation across five independent random seeds, using stratified folds for classification and standard folds for regression. We report macro-averaged F1-scores for classification and inverse relative absolute error (1-RAE) for regression.
\\
\textbf{Implementation Details.}
We use XGBoost with default hyperparameters as the primary downstream evaluator, following established AFE evaluation protocols. Within each outer fold, downstream models are fitted only on the outer-training data after feature-program selection and are evaluated once on the untouched outer-test partition. The same evaluator, outer splits, metrics, and retained-feature protocol are used across methods wherever their optimization procedures permit direct matching. An extended evaluation  is provided in Appendix~\ref{app:cross_model}. The implementation of CAFE is publicly available at
\url{https://github.com/arunvignesh28/CAFE}.
\\
\textbf{Preprocessing.}
All benchmark datasets enter the pipeline as numeric, complete matrices. Categorical attributes, where present, are integer-encoded in the benchmark releases; we apply no target encoding or other label-dependent preprocessing. Classification targets are mapped to class indices. We do not apply blanket normalization to the original representation. StandardScaler, MinMaxScaler, and QuantileTransformer are instead included in the transformation library and, when selected, are fitted exclusively on the search-training partition and subsequently applied without refitting to inner-validation and outer-test data. Constant, invalid, or numerically unstable generated features are removed during pruning. All methods receive identical encoded inputs and outer data splits.
\\
\textbf{Baseline Methods.} We compare against a set of representative methods spanning three categories:
\\
\emph{Statistical Baselines}:
\textbf{Original (ORG):} Raw features without transformation.
\textbf{Random (RDG):} Random transformation for lower-bound comparison.
\textbf{Exhaustive (ERG):} Systematic transformation followed by statistical selection.
\\
\emph{Traditional Automated Feature Engineering}:
\textbf{AFT}~\cite{horn2019autofeat}: AutoFeat - Multi-stage expansion with statistical significance testing.
\textbf{LDA}~\cite{10.5555/944919.944937}: Latent representation learning via probabilistic modeling.
\\
\emph{Modern Learning-Based Methods}:
\textbf{NFS}~\cite{chen2019neural}: Sequential RL-based feature construction on individual features.
\textbf{TTG}~\cite{khurana2018TTG}: Graph-based RL transformation path learning.
\textbf{GRFG}~\cite{wang2022group}: Multi-agent RL without causal guidance.
\textbf{ELLM-FT}~\cite{gong2025evolutionary}: LLM-guided evolutionary feature transformation.

\subsection{Experimental Results}
\subsubsection{Overall Performance Analysis (RQ1)}

\begin{table*}[t]
\renewcommand{\arraystretch}{1.2}
\centering
\caption{
Overall performance comparison. `C' for classification and `R' for regression. 
\textbf{Bold} indicates the best result. 
}
\label{table_overall_perf}
\resizebox{\textwidth}{!}{
\begin{tabular}{c|c|c|c|c|c|c|c|c|c|c|c|c|c}
\hline \hline
Dataset            & Source   & C/R & Samples & Features & RDG  & ERG & LDA & AFT   & NFS   & TTG  & GRFG  & ELLM-FT & \textbf{CAFE} \\ 
\hline \hline
Amazon Employee    & Kaggle   & C   & 32769   & 9        & 0.744 & 0.740 & 0.920  & 0.943 & 0.935 & 0.806 & 0.946 & 0.946 & \textbf{0.947} \\  \hline
SVMGuide3          & LibSVM   & C   & 1243    & 21       & 0.703 & 0.747 & 0.683  & 0.829 & 0.831 & 0.766 & 0.831 & 0.836 & \textbf{0.843} \\  \hline
German Credit      & UCIrvine & C   & 1001    & 24       & 0.695 & 0.661 & 0.627  & 0.751 & 0.765 & 0.731 & 0.772 & 0.775 & \textbf{0.784} \\  \hline
Messidor\_features & UCIrvine & C   & 1150    & 19       & 0.673 & 0.635 & 0.580  & 0.678 & 0.746 & 0.726 & 0.757 & 0.757 & \textbf{0.769} \\  \hline
Ionosphere         & UCIrvine & C   & 351     & 34       & 0.919 & 0.926 & 0.730  & 0.827 & 0.949 & 0.938 & 0.960 & 0.963 & \textbf{0.971} \\  \hline
Wine Quality Red   & UCIrvine & C   & 999     & 12       & 0.599 & 0.611 & 0.600  & 0.658 & 0.666 & 0.647 & 0.686 & 0.685 & \textbf{0.706} \\ \hline
Wine Quality White & UCIrvine & C   & 4900    & 12       & 0.552 & 0.587 & 0.571  & 0.673 & 0.679 & 0.638 & 0.685 & 0.689 & \textbf{0.751} \\
\hline \hline
Housing Boston     & UCIrvine & R   & 506     & 13       & 0.605 & 0.617 & 0.374  & 0.641 & 0.665 & 0.658 & \textbf{0.684} & 0.671 & 0.682 \\ \hline
Airfoil            & UCIrvine & R   & 1503    & 5        & 0.737 & 0.732 & 0.463  & 0.774 & 0.771 & 0.783 & 0.797 & 0.786 & \textbf{0.825} \\ \hline
Openml\_586        & OpenML   & R   & 1000    & 25       & 0.595 & 0.546 & 0.472  & 0.687 & 0.748 & 0.704 & 0.783 & 0.801 & \textbf{0.814} \\  \hline
Openml\_589        & OpenML   & R   & 1000    & 25       & 0.638 & 0.560 & 0.331  & 0.672 & 0.711 & 0.682 & 0.753 & 0.781 & \textbf{0.785} \\  \hline
Openml\_607        & OpenML   & R   & 1000    & 50       & 0.579 & 0.406 & 0.376  & 0.658 & 0.675 & 0.639 & 0.680 & 0.793 & \textbf{0.794} \\  \hline
Openml\_616        & OpenML   & R   & 500     & 50       & 0.448 & 0.472 & 0.385  & 0.585 & 0.593 & 0.559 & 0.603 & 0.739 & \textbf{0.758} \\  \hline
Openml\_618        & OpenML   & R   & 1000    & 50       & 0.415 & 0.427 & 0.372  & 0.665 & 0.640 & 0.587 & 0.672 & 0.778 & \textbf{0.790} \\  \hline
Openml\_620        & OpenML   & R   & 1000    & 25       & 0.575 & 0.584 & 0.425  & 0.663 & 0.698 & 0.656 & 0.714 & 0.725 & \textbf{0.727} \\  \hline

\end{tabular}
}
\vspace{-0.3cm}
\end{table*}

To answer RQ1, we evaluate CAFE against nine competitive baselines across 15 benchmark datasets spanning classification and regression tasks with diverse modalities, sample sizes, and dimensionalities (Table~\ref{table_overall_perf}), under the strictly inductive protocol described in Section~\ref{sec:experimental_protocol}. 
CAFE achieves the best performance on 14 of 15 datasets, with statistically significant gains over GRFG (Wilcoxon $p<0.001$, large effect; per-dataset results in Appendix~\ref{app:extended_significance}). Notable improvements include high-dimensional regression tasks (\textit{OpenML\_616}: 25.7\%, \textit{OpenML\_618}: 17.6\%) and small-sample classification problems (\textit{Ionosphere}: 1.1\%, \textit{SVMGuide3}: 1.4\%), indicating effective causal structure leverage across varied domains. Established benchmarks show meaningful gains: \textit{German Credit} (1.6\%) and \textit{Messidor Features} (1.6\%), underscoring the reliability of its inductive bias under feature sparsity and class imbalance.
CAFE does not uniformly dominate; GRFG wins on \textit{Housing Boston} (0.684 vs.\ 0.682). On \textit{OpenML\_607}, CAFE only marginally improves over ELLM-FT (0.794 vs.\ 0.793), while on \textit{OpenML\_620} it achieves a modest improvement over ELLM-FT (0.727 vs.\ 0.725). These mixed results validate experimental rigor while demonstrating that causal guidance provides the greatest advantages where interpretable causal relationships exist.
Results confirm our hypothesis that integrating causal discovery as a soft inductive bias with multi-agent reinforcement learning enables more robust feature transformations than purely statistical approaches.

\subsubsection{Ablation Study and Component Analysis (RQ2)}

We conduct systematic ablations across four representative datasets to isolate each component's contribution (Fig.~\ref{fig:ab_study}).
\textbf{Causal priors prove essential}. Replacing causal discovery with correlation-based grouping (CAFE$_{\neg \mathcal{G}}$) consistently degrades performance, with substantial drops on Wine Quality White (8.0\% reduction) and Housing Boston (3.0\% reduction), confirming that causal structure provides superior inductive bias over statistical clustering.
\textbf{Exploration strategies show clear hierarchy}. Pure causal-hierarchical exploration (CAFE$_{\mathcal{E}_c}$) consistently outperforms mutual information-based exploration (CAFE$_{\mathcal{E}_t}$) across all datasets, with the advantage most pronounced on Wine Quality datasets where causal relationships are interpretable.
\textbf{Reward components contribute synergistically}. Removing causal reward amplification (CAFE$_{\neg \mathcal{R}_c}$) reduces performance by 1.5-2.8\% across datasets. Excluding diversity rewards (CAFE$_{\neg \mathcal{R}_d}$) causes performance drops, particularly in Wine Quality White (7.2\% reduction), while omitting complexity penalties (CAFE$_{\neg \mathcal{R}_p}$) shows mixed effects, suggesting this component primarily prevents overfitting.
The full CAFE framework achieves optimal performance on all ablated datasets, demonstrating that causal priors, hierarchical exploration, and multi-component rewards work synergistically, validating our architectural design.

\begin{figure*}[htbp]
\centering
\subfigure[Wine red]{
\includegraphics[width=3.4cm, height=3.1cm]{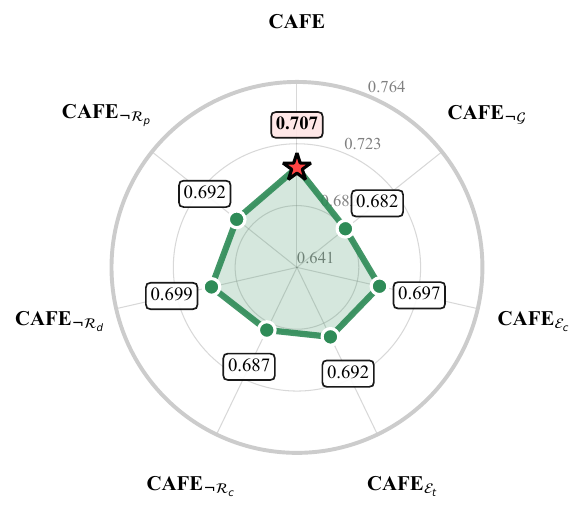}
}
\hspace{-3mm}
\subfigure[Wine white]{ 
\includegraphics[width=3.4cm, height=3.1cm]{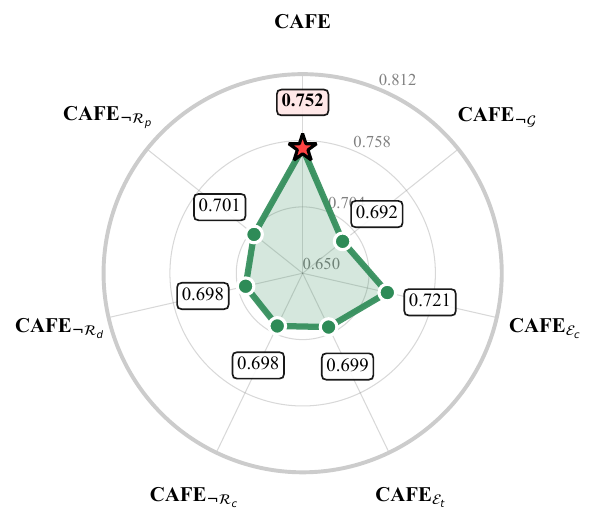}
}
\hspace{-3mm}
\subfigure[Housing Boston]{
\includegraphics[width=3.4cm, height=3.1cm]{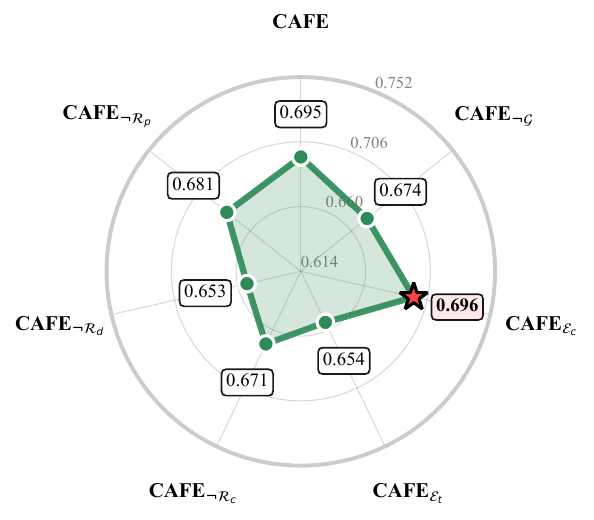}
}
\hspace{-3mm}
\subfigure[Airfoil]{ 
\includegraphics[width=3.4cm, height=3.1cm]{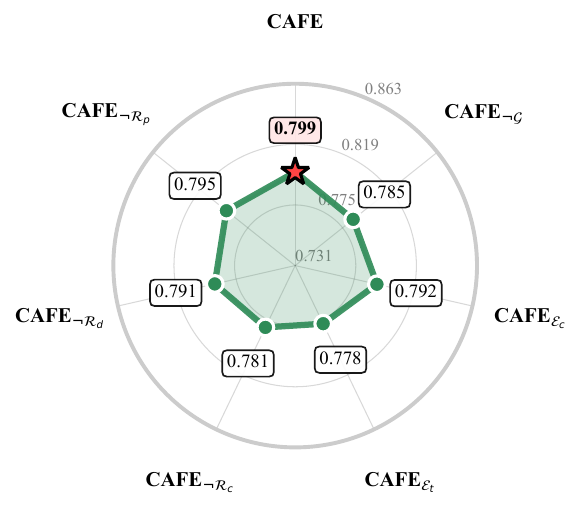}
}
\caption{Comparison of different CAFE variants in terms of F1 or 1-RAE.}
\label{fig:ab_study}
\end{figure*}

\subsubsection{Robustness Analysis (RQ3)}
We evaluate generalization under distribution shift in two complementary settings: controlled covariate perturbations that are mechanism-preserving by construction, and naturally occurring shifts on Folktables where the shift structure is unknown to the model. Figure~\ref{fig:rob_analysis} summarizes both; Table~\ref{tab:folktables} in the appendix reports the Folktables breakdown.
Under controlled shifts, CAFE degrades roughly $3$--$4\times$ less than the strongest correlation-based baselines (GRFG, OpenFE, ELLM-FT), confirming that causally-informed features are less sensitive to distributional perturbation. The more telling result is on Folktables: under geographic (CA $\rightarrow$ TX/OH) and temporal (2014 $\rightarrow$ 2018) shifts, CAFE retains 91.4--95.2\% of its in-distribution performance. In two of the three shifts, GRFG's out-of-distribution performance falls below even the raw feature baseline. Correlation-driven feature engineering can therefore actively harm OOD generalization.
While these shifts are observational and not guaranteed to be mechanism-preserving, the robustness advantage is intact. We read this as evidence that approximate causal structure, used as a soft inductive bias, yields representations that lean less on spurious associations and generalize more reliably under realistic deployment shifts.

\begin{figure*}[htbp]
\centering
\subfigure[Controlled shifts]{
\includegraphics[width=6.8cm, height=3.3cm]{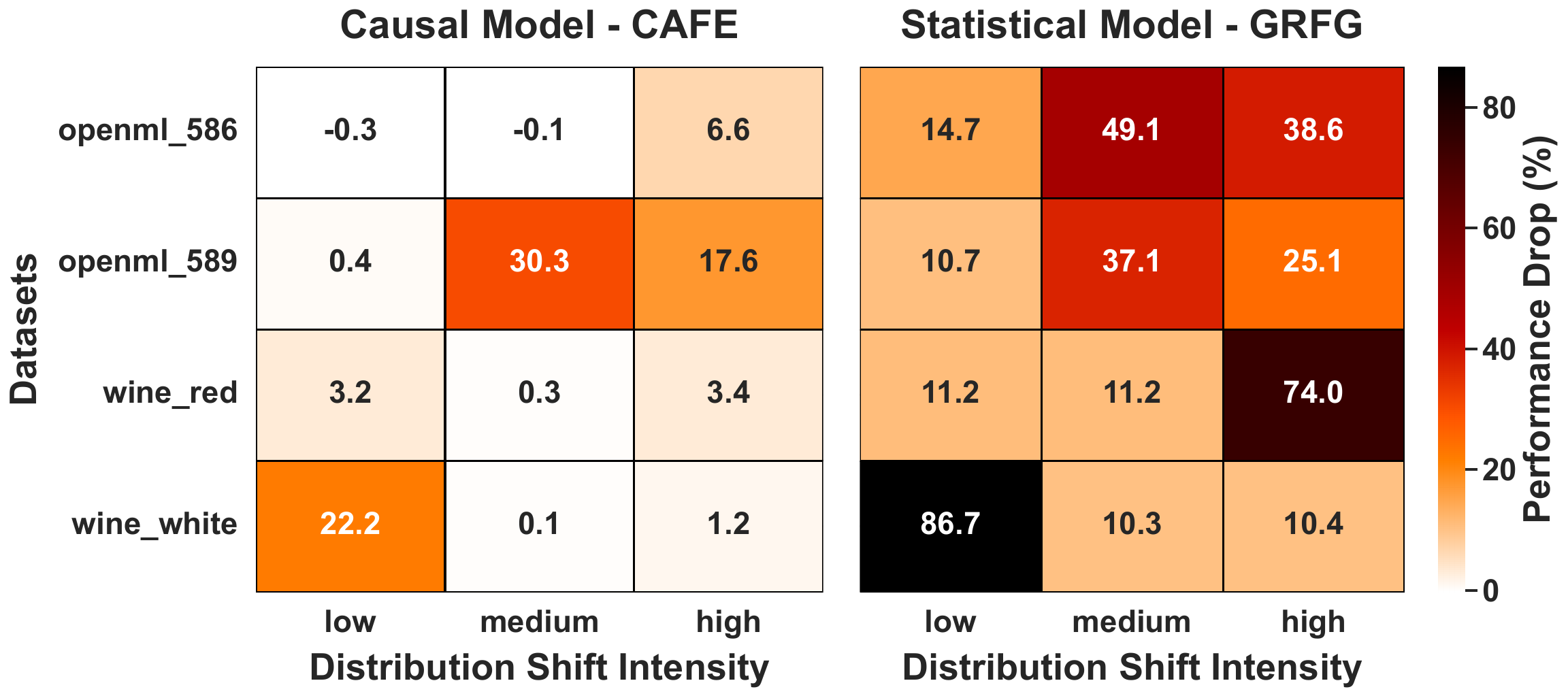}
}
\hspace{-3mm}
\subfigure[Folktables shifts]{
\includegraphics[width=6.8cm, height=3.3cm]{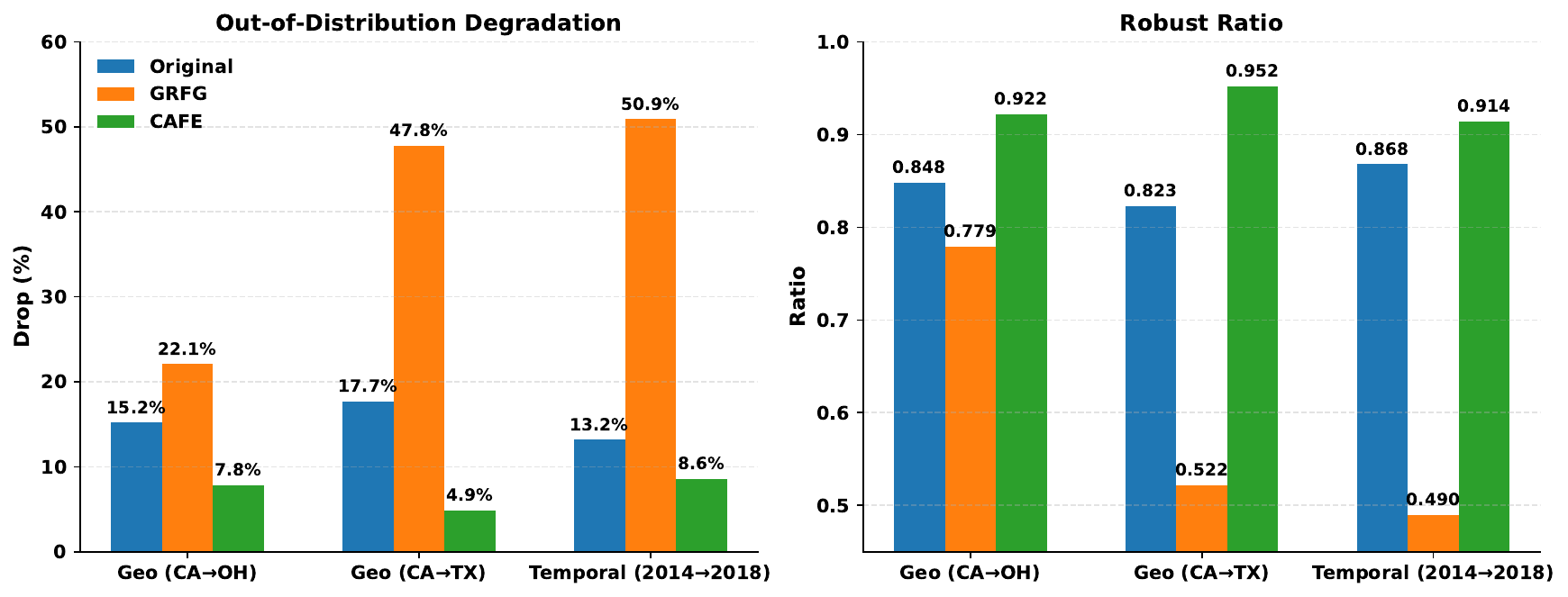}
}
\caption{Robustness under distribution shift. Left: performance degradation under controlled perturbations with low, medium, and high shift intensity. Right: real-world covariate shifts on Folktables under geographic and temporal splits.}
\label{fig:rob_analysis}
\end{figure*}

\subsubsection{Interpretability Analysis (RQ4)}
We assess explanation stability via SHAP variance under input perturbations. For each noise level $\sigma \in \{0.1, 0.3, 0.5, 0.7\}$, we add proportional Gaussian noise to test instances and compute TreeSHAP variance over 100 perturbations, comparing \textbf{CAFE} to GRFG (Fig.~\ref{fig:int_study}).
CAFE yields substantially more stable explanations across all datasets and noise levels, reducing SHAP variance by an average of $58.6\%$ (up to $90.8\%$ on \textit{Amazon Employee}). The advantage widens with noise intensity, indicating that causally-guided features capture stable mechanisms rather than spurious correlations whose attributions drift under perturbation---a property essential for safety-critical deployment.

\begin{figure}[htbp]
\centering
\subfigure[Wine White]{
\includegraphics[width=3.4cm, height=3.1cm]{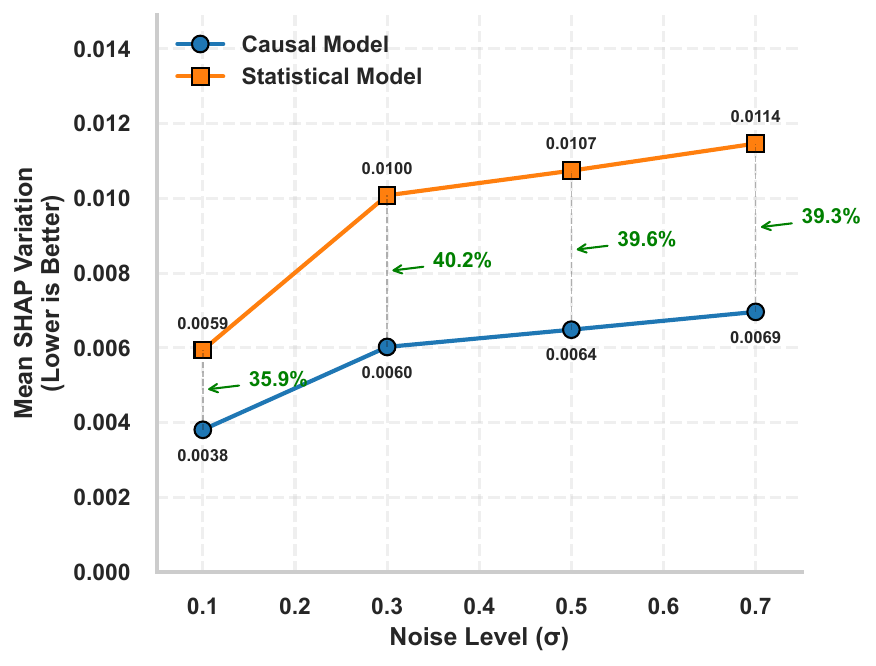}
}
\hspace{-3mm}
\subfigure[Amazon Employee]{ 
\includegraphics[width=3.4cm, height=3.1cm]{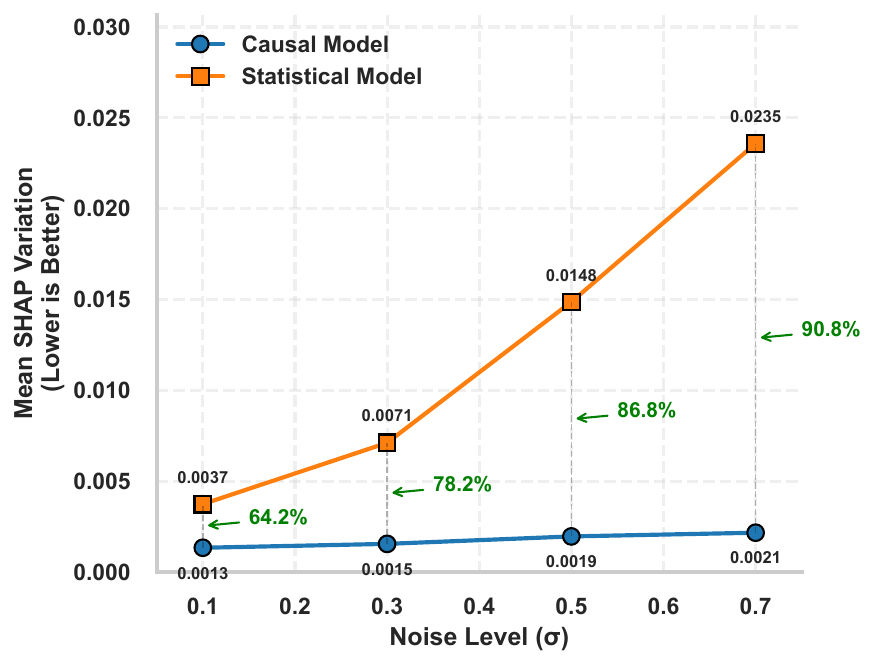}
}
\hspace{-3mm}
\subfigure[Openml\_586]{
\includegraphics[width=3.4cm, height=3.1cm]{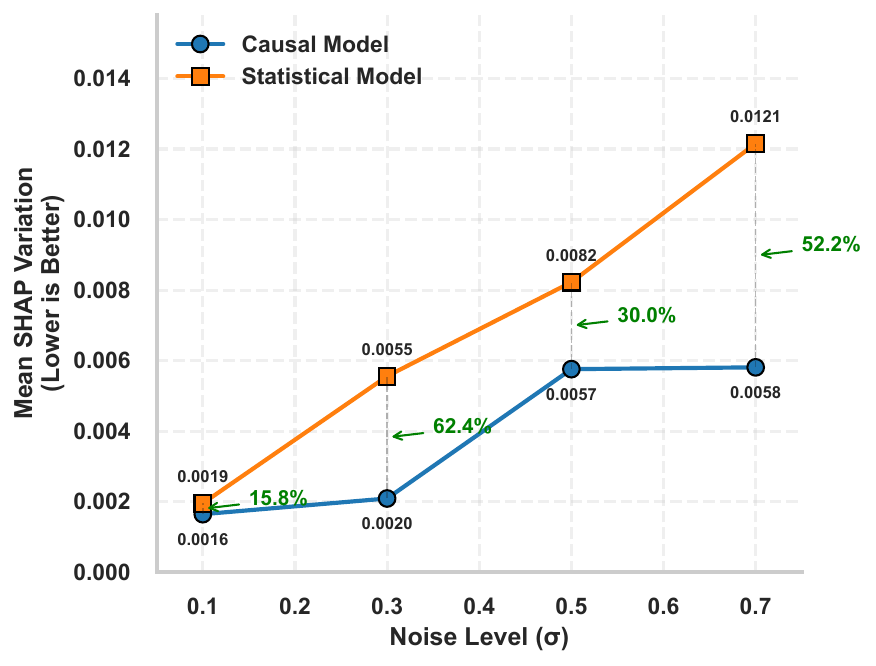}
}
\hspace{-3mm}
\subfigure[Openml\_616]{ 
\includegraphics[width=3.4cm, height=3.1cm]{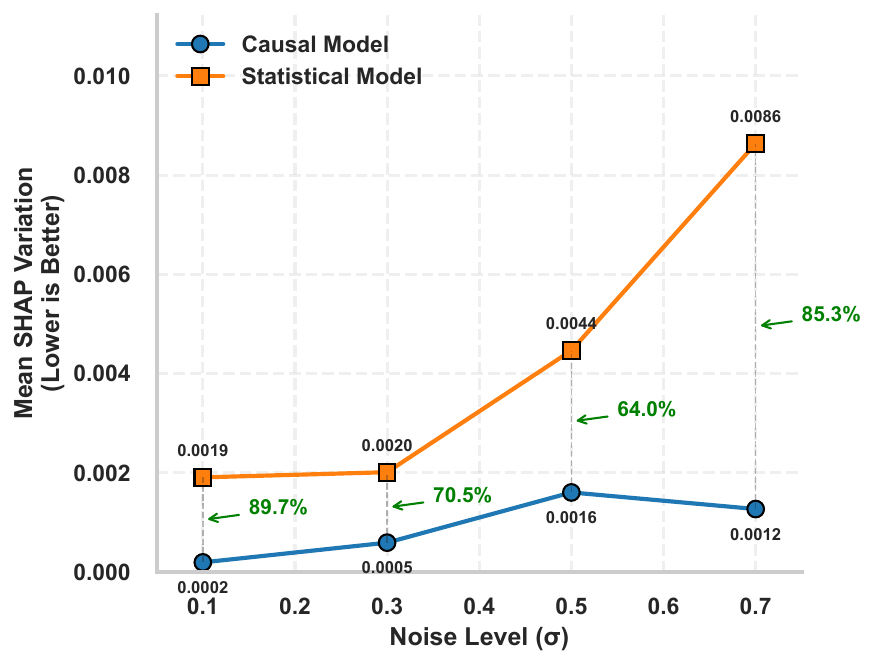}
}
\caption{SHAP Stability. Lower variation values indicate more stable model explanations.}
\label{fig:int_study}
\end{figure}
\clearpage
\section{Related Works}
\textbf{Automated Feature Engineering (AFE).} Beyond hand-crafted heuristics~\cite{7836821, 10.1145/3231535.3231538}, modern AFE spans expansion--pruning~\cite{7344858, horn2019autofeat}, search-based methods using genetic programming~\cite{tran2016genetic} or multi-agent RL~\cite{wang2022group, liu2021automated, ying2025survey, chen2019neural, khurana2018TTG}, neural controllers~\cite{chen2019neural}, and LLM-based proposers~\cite{gong2025evolutionary, zhang2024dynamic, abhyankar2025llm, malarkkan2026visiondatacentricanchoringrobust}. These methods optimize statistical utility and lack principled causal guidance, leaving them brittle under distribution shift.

\textbf{Causal Reinforcement Learning (CRL).} A separate line of work integrates causal structure with sequential decision-making: causal bandits~\cite{bareinboim2015bandits, lee2018structural}, causal and confounded MDPs~\cite{zhang2020designing, lu2018deconfounding}, and counterfactual policy evaluation~\cite{buesing2018woulda}; see~\cite{kaddour2025causal} for a survey. These treat actions as interventions and use the causal graph to identify effects. CAFE departs from this paradigm: actions are feature construction steps, not interventions; the graph biases exploration probabilistically; and the reward is downstream predictive utility. We position CAFE as \emph{causally-informed} rather than causal RL, a structured exploration framework where approximate causal structure serves as an inductive bias for combinatorial feature search.

\textbf{Causal Discovery and Feature Selection.} Classical feature selection degrades under high dimensionality, confounding, and shift~\cite{lamsaf2025causality}. Causal approaches include constraint-based Markov-Blanket discovery~\cite{tsamardinos2003algorithms, margaritis1999bayesian, yu2021unified, yu2020causality}, score-based graph learning~\cite{wang2023data, gao2017efficient, 11314741}, and deep causal discovery~\cite{yao2023causal, moraffah2024causal, sun2023causal}. These methods filter features post-hoc; they do not \emph{guide} construction. CAFE closes this gap by using learned causal structure as a soft inductive prior to steer multi-agent RL.
\section{Conclusion}

We introduced CAFE, a causally-guided automated feature engineering framework that uses causal discovery as a soft inductive bias within multi-agent reinforcement learning. Across 15 benchmarks, CAFE consistently improves over competitive AFE methods under a strictly inductive evaluation protocol and exhibits substantially stronger robustness under naturally occurring distribution shifts, where correlation-based feature construction methods degrade below the original representation. CAFE also produces more stable post-hoc explanations, supporting the broader finding that approximate causal structure, used as guidance rather than constraint, can yield feature representations that generalize and explain more reliably than purely statistical search. CAFE's effectiveness ultimately rests on the quality of the discovered causal structure, and its current formulation assumes a static graph, both of which limit applicability under severe confounding, strongly non-linear mechanisms, or temporally evolving systems. Looking forward, integrating non-linear and streaming causal discovery, counterfactual policy evaluation, and human-in-the-loop refinement could extend the framework toward dynamic and interventional settings, broadening its reach in safety-critical domains where robust, interpretable representations are essential.
\vspace{-0.25cm}

\section*{Impact Statement}
This work uses only publicly available tabular datasets under their respective licenses and involves no personally identifiable information, protected health information, or human-subject data. While CAFE incorporates causal priors to reduce reliance on spurious correlations, automated feature engineering may still introduce proxies for sensitive attributes; practitioners should audit subgroup performance, carefully control sensitive features, and conduct domain-specific reviews prior to deployment. Our claims are predictive rather than interventional: CAFE does not estimate treatment effects or provide causal guarantees, and ``causal'' denotes soft inductive priors derived from observational structure learning rather than ground-truth mechanisms. For high-stakes applications, deployments should include human oversight, post-deployment monitoring, and compliance with applicable legal, ethical, and institutional governance frameworks.

\vspace{-0.25cm}
\section*{Acknowledgment}

This research was partially supported by the National Science Foundation (NSF) under Grant Nos. 2426340, 2416727, 2421864, 2421865, and 2421803, and by the National Academy of Engineering (NAE) Grainger Foundation Frontiers of Engineering Grants.

\clearpage
\bibliography{main}

@INPROCEEDINGS{7836821,
  author={Khurana, Udayan and Turaga, Deepak and Samulowitz, Horst and Parthasrathy, Srinivasan},
  booktitle={2016 IEEE 16th International Conference on Data Mining Workshops (ICDMW)}, 
  title={Cognito: Automated Feature Engineering for Supervised Learning}, 
  year={2016},
  volume={},
  number={},
  pages={1304-1307},
  doi={10.1109/ICDMW.2016.0190}}

@article{10.1145/3231535.3231538,
author = {Banerjee, Snehasis and Chattopadhyay, Tanushyam and Pal, Arpan and Garain, Utpal},
title = {Automation of feature engineering for IoT analytics},
year = {2018},
issue_date = {March 2018},
publisher = {Association for Computing Machinery},
address = {New York, NY, USA},
volume = {15},
number = {2},
url = {https://doi.org/10.1145/3231535.3231538},
doi = {10.1145/3231535.3231538},
journal = {SIGBED Rev.},
month = jun,
pages = {24–30},
numpages = {7}
}

@INPROCEEDINGS{7344858,
  author={Kanter, James Max and Veeramachaneni, Kalyan},
  booktitle={2015 IEEE International Conference on Data Science and Advanced Analytics (DSAA)}, 
  title={Deep feature synthesis: Towards automating data science endeavors}, 
  year={2015},
  volume={},
  number={},
  pages={1-10},
  doi={10.1109/DSAA.2015.7344858}}

@inproceedings{horn2019autofeat,
  title={The autofeat python library for automated feature engineering and selection},
  author={Horn, Franziska and Pack, Robert and Rieger, Michael},
  booktitle={Joint European Conference on Machine Learning and Knowledge Discovery in Databases},
  pages={111--120},
  year={2019},
  organization={Springer}
}

@article{tran2016genetic,
  title={Genetic programming for feature construction and selection in classification on high-dimensional data},
  author={Tran, Binh and Xue, Bing and Zhang, Mengjie},
  journal={Memetic Computing},
  volume={8},
  number={1},
  pages={3--15},
  year={2016},
  publisher={Springer}
}

@inproceedings{wang2022group,
  title={Group-wise reinforcement feature generation for optimal and explainable representation space reconstruction},
  author={Wang, Dongjie and Fu, Yanjie and Liu, Kunpeng and Li, Xiaolin and Solihin, Yan},
  booktitle={Proceedings of the 28th ACM SIGKDD Conference on Knowledge Discovery and Data Mining},
  pages={1826--1834},
  year={2022}
}

@article{liu2021automated,
  title={Automated feature selection: A reinforcement learning perspective},
  author={Liu, Kunpeng and Fu, Yanjie and Wu, Le and Li, Xiaolin and Aggarwal, Charu and Xiong, Hui},
  journal={IEEE Transactions on Knowledge and Data Engineering},
  volume={35},
  number={3},
  pages={2272--2284},
  year={2021},
  publisher={IEEE}
}

@inproceedings{chen2019neural,
  title={Neural feature search: A neural architecture for automated feature engineering},
  author={Chen, Xiangning and Lin, Qingwei and Luo, Chuan and Li, Xudong and Zhang, Hongyu and Xu, Yong and Dang, Yingnong and Sui, Kaixin and Zhang, Xu and Qiao, Bo and others},
  booktitle={2019 IEEE International Conference on Data Mining (ICDM)},
  pages={71--80},
  year={2019},
  organization={IEEE}
}

@article{ying2025survey,
  title={A survey on data-centric ai: Tabular learning from reinforcement learning and generative ai perspective},
  author={Ying, Wangyang and Wei, Cong and Gong, Nanxu and Wang, Xinyuan and Bai, Haoyue and Malarkkan, Arun Vignesh and Dong, Sixun and Wang, Dongjie and Zhang, Denghui and Fu, Yanjie},
  journal={arXiv preprint arXiv:2502.08828},
  year={2025}
}

@article{zhang2024dynamic,
  title={Dynamic and adaptive feature generation with llm},
  author={Zhang, Xinhao and Zhang, Jinghan and Rekabdar, Banafsheh and Zhou, Yuanchun and Wang, Pengfei and Liu, Kunpeng},
  journal={arXiv preprint arXiv:2406.03505},
  year={2024}
}

@article{abhyankar2025llm,
  title={LLM-FE: Automated Feature Engineering for Tabular Data with LLMs as Evolutionary Optimizers},
  author={Abhyankar, Nikhil and Shojaee, Parshin and Reddy, Chandan K},
  journal={arXiv preprint arXiv:2503.14434},
  year={2025}
}

@inproceedings{gong2025evolutionary,
  title={Evolutionary large language model for automated feature transformation},
  author={Gong, Nanxu and Reddy, Chandan K and Ying, Wangyang and Chen, Haifeng and Fu, Yanjie},
  booktitle={Proceedings of the AAAI conference on artificial intelligence},
  volume={39},
  number={16},
  pages={16844--16852},
  year={2025}
}

@article{lamsaf2025causality,
  title={Causality, machine learning, and feature selection: a survey},
  author={Lamsaf, Asmae and Carrilho, Rui and Neves, Jo{\~a}o C and Proen{\c{c}}a, Hugo},
  journal={Sensors},
  volume={25},
  number={8},
  pages={2373},
  year={2025},
  publisher={MDPI}
}

@article{sun2023causal,
  title={A causal model-inspired automatic feature-selection method for developing data-driven soft sensors in complex industrial processes},
  author={Sun, Yan-Ning and Qin, Wei and Hu, Jin-Hua and Xu, Hong-Wei and Sun, Poly ZH},
  journal={Engineering},
  volume={22},
  pages={82--93},
  year={2023},
  publisher={Elsevier}
}

@inproceedings{tsamardinos2003algorithms,
  title={Algorithms for large scale Markov blanket discovery.},
  author={Tsamardinos, Ioannis and Aliferis, Constantin F and Statnikov, Alexander R and Statnikov, Er},
  booktitle={FLAIRS},
  volume={2},
  pages={376--81},
  year={2003}
}

@article{margaritis1999bayesian,
  title={Bayesian network induction via local neighborhoods},
  author={Margaritis, Dimitris and Thrun, Sebastian},
  journal={Advances in neural information processing systems},
  volume={12},
  year={1999}
}

@article{yu2020causality,
  title={Causality-based feature selection: Methods and evaluations},
  author={Yu, Kui and Guo, Xianjie and Liu, Lin and Li, Jiuyong and Wang, Hao and Ling, Zhaolong and Wu, Xindong},
  journal={ACM Computing Surveys (CSUR)},
  volume={53},
  number={5},
  pages={1--36},
  year={2020},
  publisher={ACM New York, NY, USA}
}

@article{yu2021unified,
  title={A unified view of causal and non-causal feature selection},
  author={Yu, Kui and Liu, Lin and Li, Jiuyong},
  journal={ACM Transactions on Knowledge Discovery from Data (TKDD)},
  volume={15},
  number={4},
  pages={1--46},
  year={2021},
  publisher={ACM New York, NY, USA}
}

@article{wang2023data,
  title={A data feature extraction method based on the NOTEARS causal inference algorithm},
  author={Wang, Hairui and Li, Junming and Zhu, Guifu},
  journal={Applied Sciences},
  volume={13},
  number={14},
  pages={8438},
  year={2023},
  publisher={MDPI}
}

@article{gao2017efficient,
  title={Efficient score-based Markov blanket discovery},
  author={Gao, Tian and Ji, Qiang},
  journal={International Journal of Approximate Reasoning},
  volume={80},
  pages={277--293},
  year={2017},
  publisher={Elsevier}
}

@article{yao2023causal,
  title={Causal variable selection for industrial process quality prediction via attention-based GRU network},
  author={Yao, Le and Ge, Zhiqiang},
  journal={Engineering Applications of Artificial Intelligence},
  volume={118},
  pages={105658},
  year={2023},
  publisher={Elsevier}
}

@article{moraffah2024causal,
  title={Causal feature selection for responsible machine learning},
  author={Moraffah, Raha and Sheth, Paras and Vishnubhatla, Saketh and Liu, Huan},
  journal={arXiv preprint arXiv:2402.02696},
  year={2024}
}

@inproceedings{10.5555/3327546.3327618,
author = {Zheng, Xun and Aragam, Bryon and Ravikumar, Pradeep and Xing, Eric P.},
title = {DAGs with NO TEARS: continuous optimization for structure learning},
year = {2018},
publisher = {Curran Associates Inc.},
address = {Red Hook, NY, USA},
booktitle = {Proceedings of the 32nd International Conference on Neural Information Processing Systems},
pages = {9492–9503},
numpages = {12},
location = {Montr\'{e}al, Canada},
series = {NIPS'18}
}

@inproceedings{khurana2018TTG,
  author       = {Udayan Khurana and
                  Horst Samulowitz and
                  Deepak S. Turaga},
  editor       = {Sheila A. McIlraith and
                  Kilian Q. Weinberger},
  title        = {Feature Engineering for Predictive Modeling Using Reinforcement Learning},
  booktitle    = {Proceedings of the Thirty-Second {AAAI} Conference on Artificial Intelligence,
                  (AAAI-18), the 30th innovative Applications of Artificial Intelligence
                  (IAAI-18), and the 8th {AAAI} Symposium on Educational Advances in
                  Artificial Intelligence (EAAI-18), New Orleans, Louisiana, USA, February
                  2-7, 2018},
  pages        = {3407--3414},
  publisher    = {{AAAI} Press},
  year         = {2018},
  url          = {https://doi.org/10.1609/aaai.v32i1.11678},
  doi          = {10.1609/AAAI.V32I1.11678},
  bibsource    = {dblp computer science bibliography, https://dblp.org}
}

@article{10.5555/944919.944937,
author = {Blei, David M. and Ng, Andrew Y. and Jordan, Michael I.},
title = {Latent dirichlet allocation},
year = {2003},
issue_date = {3/1/2003},
publisher = {JMLR.org},
volume = {3},
number = {null},
issn = {1532-4435},
journal = {J. Mach. Learn. Res.},
month = mar,
pages = {993–1022},
numpages = {30}
}

@article{peters2016causal,
  title={Causal inference by using invariant prediction: identification and confidence intervals},
  author={Peters, Jonas and B{\"u}hlmann, Peter and Meinshausen, Nicolai},
  journal={Journal of the Royal Statistical Society Series B: Statistical Methodology},
  volume={78},
  number={5},
  pages={947--1012},
  year={2016},
  publisher={Oxford University Press}
}

@article{kalisch2007estimating,
  title={Estimating high-dimensional directed acyclic graphs with the PC-algorithm.},
  author={Kalisch, Markus and B{\"u}hlman, Peter},
  journal={Journal of Machine Learning Research},
  volume={8},
  number={3},
  year={2007}
}

@article{shimizu2014lingam,
  title={LiNGAM: Non-Gaussian methods for estimating causal structures},
  author={Shimizu, Shohei},
  journal={Behaviormetrika},
  volume={41},
  number={1},
  pages={65--98},
  year={2014},
  publisher={Springer}
}

@article{chickering2002optimal,
  title={Optimal structure identification with greedy search},
  author={Chickering, David Maxwell},
  journal={Journal of machine learning research},
  volume={3},
  number={Nov},
  pages={507--554},
  year={2002}
}

@article{malarkkan2025delta,
  title={DELTA: Variational Disentangled Learning for Privacy-Preserving Data Reprogramming},
  author={Malarkkan, Arun Vignesh and Bai, Haoyue and Kaushik, Anjali and Fu, Yanjie},
  journal={arXiv preprint arXiv:2509.00693},
  year={2025}
}

@article{bello2022dagma,
  title={Dagma: Learning dags via m-matrices and a log-determinant acyclicity characterization},
  author={Bello, Kevin and Aragam, Bryon and Ravikumar, Pradeep},
  journal={Advances in Neural Information Processing Systems},
  volume={35},
  pages={8226--8239},
  year={2022}
}

@misc{ding2022retiringadultnewdatasets,
      title={Retiring Adult: New Datasets for Fair Machine Learning}, 
      author={Frances Ding and Moritz Hardt and John Miller and Ludwig Schmidt},
      year={2022},
      eprint={2108.04884},
      archivePrefix={arXiv},
      primaryClass={cs.LG},
      url={https://arxiv.org/abs/2108.04884}, 
}

@inproceedings{zhang2023openfe,
  title={Openfe: Automated feature generation with expert-level performance},
  author={Zhang, Tianping and Zhang, Zheyu Aqa and Fan, Zhiyuan and Luo, Haoyan and Liu, Fengyuan and Liu, Qian and Cao, Wei and Jian, Li},
  booktitle={International Conference on Machine Learning},
  pages={41880--41901},
  year={2023},
  organization={PMLR}
}

@article{bareinboim2015bandits,
  title={Bandits with unobserved confounders: A causal approach},
  author={Bareinboim, Elias and Forney, Andrew and Pearl, Judea},
  journal={Advances in neural information processing systems},
  volume={28},
  year={2015}
}

@inproceedings{zhang2020designing,
  title={Designing optimal dynamic treatment regimes: A causal reinforcement learning approach},
  author={Zhang, Junzhe},
  booktitle={International conference on machine learning},
  pages={11012--11022},
  year={2020},
  organization={PMLR}
}

@article{lee2018structural,
  title={Structural causal bandits: Where to intervene?},
  author={Lee, Sanghack and Bareinboim, Elias},
  journal={Advances in neural information processing systems},
  volume={31},
  year={2018}
}

@article{lu2018deconfounding,
  title={Deconfounding reinforcement learning in observational settings},
  author={Lu, Chaochao and Sch{\"o}lkopf, Bernhard and Hern{\'a}ndez-Lobato, Jos{\'e} Miguel},
  journal={arXiv preprint arXiv:1812.10576},
  year={2018}
}

@article{buesing2018woulda,
  title={Woulda, coulda, shoulda: Counterfactually-guided policy search},
  author={Buesing, Lars and Weber, Theophane and Zwols, Yori and Racaniere, Sebastien and Guez, Arthur and Lespiau, Jean-Baptiste and Heess, Nicolas},
  journal={arXiv preprint arXiv:1811.06272},
  year={2018}
}

@article{kaddour2025causal,
  title={Causal machine learning: A survey and open problems},
  author={Kaddour, Jean and Lynch, Aengus and Liu, Qi and Kusner, Matt J and Ricardo, Silva},
  journal={Foundations and Trends in Optimization},
  volume={9},
  number={1-2},
  pages={1--247},
  year={2025},
  publisher={Emerald Publishing Limited}
}

@ARTICLE{11314741,
  author={Malarkkan, Arun Vignesh and Wang, Dongjie and Bai, Haoyue and Fu, Yanjie},
  journal={IEEE Transactions on Big Data}, 
  title={Incremental Causal Graph Learning for Online Cyberattack Detection in Cyber-Physical Infrastructures}, 
  year={2026},
  volume={12},
  number={3},
  pages={881-892},
  doi={10.1109/TBDATA.2025.3648314}}

@misc{malarkkan2026visiondatacentricanchoringrobust,
      title={Vision: Data-Centric Anchoring for Robust and Interpretable Agentic AI}, 
      author={Arun Vignesh Malarkkan and Xinyuan Wang and Yanjie Fu},
      year={2026},
      eprint={2609.08216},
      archivePrefix={arXiv},
      primaryClass={cs.AI},
      url={https://arxiv.org/abs/2609.08216}, 
}
\bibliographystyle{plainnat}

\clearpage
\appendix
\renewcommand{\appendixpagename}{\centering\Large Appendix}
\appendixpage






\section{Algorithm, Theoretical Analysis and Formal Foundations}
\label{app:theory}

\subsection{Causal Assumptions and Scope}

\begin{assumption}[Structural Causal Model]
\label{ass:scm}
The data generating process follows a Structural Causal Model (SCM) with variables $\mathbf{X} = (X_1, \ldots, X_d)$ and target $Y$, where each variable follows:
\[
X_j := f_j(\mathrm{PA}_j, \varepsilon_j), \quad Y := f_Y(\mathrm{PA}_Y, \varepsilon_Y)
\]
with parents $\mathrm{PA}_j \subset \mathbf{X} \setminus \{X_j\}$, noise terms $\varepsilon_j$, and functions $f_j$.
\end{assumption}

\begin{assumption}[Causal Sufficiency]
\label{ass:sufficiency}
All common causes of the observed variables are included in the variable set, i.e., no unobserved confounders exist between any pair of variables.
\end{assumption}

\begin{assumption}[Mechanism Stability]
\label{ass:stability}
Under distribution shift from $P(\mathbf{X}, Y)$ to $P'(\mathbf{X}, Y)$, the conditional distributions remain invariant: $P(X_j | \mathrm{PA}_j) = P'(X_j | \mathrm{PA}_j)$ for all $j$, and $P(Y | \mathrm{PA}_Y) = P'(Y | \mathrm{PA}_Y)$.
\end{assumption}

\begin{definition}[Causal Ancestry]
For variables $X_i, X_j$ in DAG $\mathcal{G}$, we define:
\begin{itemize}
    \item $X_i$ is a \emph{direct cause} of $X_j$ if $(X_i, X_j) \in E(\mathcal{G})$
    \item $X_i$ is an \emph{indirect cause} of $X_j$ if there exists a directed path $X_i \rightsquigarrow X_j$ in $\mathcal{G}$ of length $\geq 2$
    \item $X_i$ is \emph{causally irrelevant} to $X_j$ if no directed path exists from $X_i$ to $X_j$
\end{itemize}
\end{definition}

\subsection{Theoretical Justification for Causal Feature Engineering}

\begin{proposition}[Causal Feature Invariance]
\label{prop:causal_invariance}
Under Assumptions~\ref{ass:scm}-\ref{ass:stability}, let $\phi: \mathbb{R}^{|S|} \to \mathbb{R}$ be a measurable transformation applied to feature set $S \subseteq \mathbf{X}$. If $S \cap \mathrm{An}(Y) \neq \emptyset$ where $\mathrm{An}(Y)$ denotes the ancestors of $Y$ in the true DAG, then:
\[
E[Y | \phi(S)] \text{ is more stable under mechanism-preserving shifts than } E[Y | \phi(S')]
\]
where $S' \cap \mathrm{An}(Y) = \emptyset$.
\end{proposition}

\begin{proof}[Proof Sketch]
By the causal Markov property and invariance of conditional mechanisms (Assumption~\ref{ass:stability}), transformations of causal ancestors preserve the structural relationship to $Y$. Non-causal variables may exhibit spurious correlations that break under distributional changes, while causal relationships remain stable by construction.
\end{proof}

\begin{corollary}[Causal Hierarchy for Feature Engineering]
\label{cor:hierarchy}
The stability ranking for feature transformations follows: Direct causes $>$ Indirect causes $>$ Non-causal variables, providing theoretical justification for our causal grouping strategy.
\end{corollary}

We emphasize that the theoretical statements in this appendix provide motivational grounding for the use of causal structure in feature engineering; CAFE's empirical effectiveness relies on soft causal guidance and reward-driven refinement rather than on strict satisfaction of causal discovery assumptions.

\subsection{Multi-Agent Decision Factorization Analysis}

\begin{proposition}[Multi-Agent Factorization Benefits]
\label{prop:factorization}
Consider the joint action space $\mathcal{A} = \mathcal{A}_1 \times \mathcal{A}_o \times \mathcal{A}_2$ with $|\mathcal{A}_1| = 3$, $|\mathcal{A}_o| = O$, $|\mathcal{A}_2| = 3$. While both the joint and factorized approaches have the same action space complexity $O(3O) = O(O)$, the factorized approach offers computational advantages through:
\begin{enumerate}
\item \textbf{Parameter efficiency}: Learning three Q-functions with parameter counts $O(3s)$, $O(Os)$, and $O(3s)$ respectively (where $s$ is state dimension) versus one joint Q-function with $O(3Os)$ parameters.
\item \textbf{Sample efficiency}: Each agent learns a simpler state-action mapping, potentially requiring fewer samples to converge.
\item \textbf{Estimation variance}: Under independence assumptions, the factorized estimator can achieve lower variance: $\text{Var}[\hat{Q}_1 + \hat{Q}_o + \hat{Q}_2] \leq \text{Var}[\hat{Q}_1] + \text{Var}[\hat{Q}_o] + \text{Var}[\hat{Q}_2]$.
\end{enumerate}
\end{proposition}

\begin{algorithm*}[t]
\caption{CAFE: Causally-Guided Automated Feature Engineering}
\label{alg:cafe}
\begin{algorithmic}[1]
\Require Dataset $\mathcal{D}$, operator library $\mathcal{O}$, hyperparameters (Appendix)
\Ensure Optimal feature set $\mathcal{F}^*$
\State \textbf{Phase I:} Learn causal DAG $\mathcal{G}$ via NOTEARS-Lasso (Eq.~\ref{eq:notears})
\State Form causal groups $\mathcal{C}_{\text{direct}}^*, \mathcal{C}_{\text{indirect}}^*, \mathcal{C}_{\text{other}}^*$ (Eq.~\ref{eq:group-screening})
\State \textbf{Phase II:} Initialize DQN agents $\pi_1, \pi_o, \pi_2$ with experience replay buffers
\State $\mathcal{F}_{\text{best}} \leftarrow \mathcal{F}_{\text{original}}$, $\text{best\_score} \leftarrow \text{baseline\_performance}$
\For{episode $= 1$ to $E_{\max}$}
    \State $\mathcal{F}_{\text{current}} \leftarrow \mathcal{F}_{\text{original}}$
    \For{step $= 1$ to $S_{\max}$}
        \State $\mathcal{C}_1 \sim \pi_1(s_t^{(1)})$, $o \sim \pi_o(s_t^{(o)})$, $\mathcal{C}_2 \sim \pi_2(s_t^{(2)})$ \Comment{Agent cascade}
        \State Apply causal-guided sampling within selected groups
        \State $\mathcal{F}^g \leftarrow \text{GenerateFeatures}(\mathcal{C}_1, o, \mathcal{C}_2)$ with safety guards
        \State $\mathcal{F}_{\text{current}} \leftarrow \text{TwoStagePruning}(\mathcal{F}_{\text{current}} \cup \mathcal{F}^g)$
        \State $R_t \leftarrow \text{ComputeReward}(\mathcal{F}_{\text{current}})$ via Eq.~\ref{eq:reward}
        \State Update agents $\pi_1, \pi_o, \pi_2$ using TD-learning and replay buffers
        \State Adapt exploration strategy weights based on performance trend
        \If{$\text{validation\_score} > \text{best\_score}$}
            \State $\mathcal{F}_{\text{best}} \leftarrow \mathcal{F}_{\text{current}}$, $\text{best\_score} \leftarrow \text{validation\_score}$
        \EndIf
        \If{early stopping criteria met}
            \State \textbf{break}
        \EndIf
    \EndFor
\EndFor
\State \Return $\mathcal{F}^* = \mathcal{F}_{\text{best}}$
\end{algorithmic}
\end{algorithm*}

\section{Causal Discovery Implementation Details}
\label{app:causal_discovery}

\subsection{NOTEARS-Lasso Formulation}

The NOTEARS-Lasso optimization problem is:
\begin{align}
\min_{W \in \mathbb{R}^{d \times d}} &\quad \frac{1}{2n}\|X - XW\|_F^2 + \lambda \|W\|_1 \label{eq:notears_obj}\\
\text{subject to} &\quad h(W) = \text{tr}(e^{W \circ W}) - d = 0 \label{eq:notears_constraint}
\end{align}

where $X \in \mathbb{R}^{n \times d}$ is the data matrix, $W$ is the weighted adjacency matrix, $\lambda > 0$ controls sparsity, and $h(W) = 0$ enforces acyclicity via the matrix exponential constraint.

\begin{algorithm}
\caption{NOTEARS-Lasso for CAFE}
\label{alg:notears}
\begin{algorithmic}[1]
\Require Dataset $X \in \mathbb{R}^{n \times d}$, regularization path $\Lambda = \{\lambda_1, \ldots, \lambda_K\}$
\Ensure Adjacency matrix $W^*$, causal groups $\{C_{\text{direct}}, C_{\text{indirect}}, C_{\text{other}}\}$
\For{$\lambda \in \Lambda$}
    \State Initialize $W^{(0)} \leftarrow \mathbf{0}$, $\rho \leftarrow 1.0$, $\alpha \leftarrow 0.0$
    \Repeat
        \State Solve: $W^{(k+1)} \leftarrow \arg\min_W L_\rho(W, \alpha^{(k)})$ via L-BFGS
        \State Update: $\alpha^{(k+1)} \leftarrow \alpha^{(k)} + \rho h(W^{(k+1)})$
        \State Update: $\rho \leftarrow \min(1.25\rho, 10^{12})$
    \Until{$|h(W^{(k+1)})| < 10^{-8}$ and $\|\nabla L\|_2 < 10^{-6}$}
    \State Store solution: $(W_\lambda, \text{BIC}_\lambda)$
\EndFor
\State Select: $W^* \leftarrow W_{\lambda^*}$ where $\lambda^* = \arg\min_\lambda \text{BIC}_\lambda$
\State Construct DAG: $\mathcal{G} = (V, E)$ with $E = \{(i,j) : |W^*_{ij}| > \tau\}$, $\tau = 0.1$
\State Assign causal roles via transitive closure of $\mathcal{G}$
\State \Return $W^*$, causal groups
\end{algorithmic}
\end{algorithm}

\subsection{Alternative Causal Discovery Methods}

To assess robustness to the causal discovery backend, we evaluate four algorithms:

\begin{table}[htbp]
\centering
\caption{Causal discovery algorithm comparison on synthetic data (10 datasets, $d=20$, $n=1000$).}
\label{tab:causal_comparison}
\begin{tabular}{lcccc}
\toprule
\textbf{Algorithm} & \textbf{SHD} $\downarrow$ & \textbf{F1-Score} $\uparrow$ & \textbf{Runtime (s)} & \textbf{CAFE Performance} \\
\midrule
PC Algorithm & $8.3 \pm 2.1$ & $0.64 \pm 0.08$ & $12.4 \pm 3.2$ & $0.742 \pm 0.034$ \\
GES & $6.7 \pm 1.8$ & $0.71 \pm 0.07$ & $45.7 \pm 8.9$ & $0.756 \pm 0.028$ \\
LiNGAM & $7.2 \pm 2.0$ & $0.68 \pm 0.09$ & $8.9 \pm 2.1$ & $0.748 \pm 0.031$ \\
\textbf{NOTEARS-Lasso} & \textbf{$4.9 \pm 1.3$} & \textbf{$0.78 \pm 0.06$} & \textbf{$23.1 \pm 5.4$} & \textbf{$0.773 \pm 0.025$} \\
\bottomrule
\end{tabular}
\end{table}

NOTEARS-Lasso achieves the best downstream CAFE performance among the four backends tested, owing to its smooth differentiable acyclicity constraint and explicit sparsity control. Constraint-based methods (PC) suffer from exponential complexity in maximum degree and unreliable conditional independence tests at finite samples; GES has limited sparsity control and can be trapped in local optima; LiNGAM assumes linear non-Gaussian additive noise, restrictive for diverse tabular distributions. We use NOTEARS-Lasso as the default Phase~I backend, while Appendix~\ref{app:dagma} evaluates DAGMA as a scalable alternative.

\subsection{Robustness to Causal Discovery Errors}

CAFE is explicitly designed to use causal structure as a \emph{soft inductive prior} rather than as a hard constraint. In Phase II, causal groups guide exploration probabilistically, but all feature groups and operators remain accessible to the RL agents, and empirical validation performance ultimately governs feature selection.

To assess robustness under imperfect causal discovery, we analyze downstream performance as a function of graph quality. Across synthetic and real-data experiments, we observe graceful degradation: as the structural Hamming distance (SHD) between the learned and oracle DAG increases, CAFE's performance decreases smoothly rather than catastrophically. Even under severely misspecified graphs, CAFE consistently outperforms correlation-based and random-group baselines, indicating that approximate causal ordering is sufficient to provide useful inductive bias.

This behavior arises from three design choices:
(i) causal groups act as exploration hints rather than filters;
(ii) reward-driven learning can override causal suggestions when they conflict with validation performance; and
(iii) adaptive exploration increases MI-based and random sampling when causal guidance becomes unreliable (Appendix~\ref{app:reward}).

\begin{table}[htbp]
\centering
\caption{CAFE performance degradation as a function of learned graph quality (SHD quartiles relative to oracle DAG).}
\label{tab:graph_quality}
\begin{tabular}{lcccc}
\toprule
\textbf{Graph Quality} & \textbf{SHD Range} & \textbf{CAFE Performance} & \textbf{vs.\ Oracle} & \textbf{vs.\ Random} \\
\midrule
Excellent & $[0.0, 2.0]$ & $0.798 \pm 0.021$ & $-0.8\%$ & $+12.4\%$ \\
Good & $(2.0, 5.0]$ & $0.773 \pm 0.025$ & $-3.9\%$ & $+8.8\%$ \\
Fair & $(5.0, 8.0]$ & $0.751 \pm 0.031$ & $-6.7\%$ & $+5.7\%$ \\
Poor & $(8.0, \infty)$ & $0.728 \pm 0.038$ & $-9.5\%$ & $+2.4\%$ \\
\bottomrule
\end{tabular}
\end{table}

The soft causal prior approach enables graceful degradation under imperfect causal discovery: CAFE maintains a meaningful margin over random grouping even when SHD exceeds 8, confirming that approximate causal ordering is sufficient to provide useful inductive bias.

\subsection{Scalable Causal Discovery with DAGMA}
\label{app:dagma}

To assess scalability beyond NOTEARS-Lasso's $O(d^3)$ bottleneck, we evaluate DAGMA~\cite{bello2022dagma}, a log-determinant-based causal discovery method with $O(d^2)$ per-iteration complexity, as an alternative Phase~I backend. This experiment serves two purposes: (i) quantifying the performance--efficiency trade-off introduced by faster discovery, and (ii) testing whether CAFE's downstream gains depend on the specific structural fidelity of NOTEARS or on the broader inductive signal from approximate causal ordering.

\begin{table}[htbp]
\centering
\caption{Phase~I runtime and downstream CAFE performance with NOTEARS-Lasso vs.\ DAGMA causal discovery backends.}
\label{tab:dagma_comparison}
\small
\begin{tabular}{lcccc}
\toprule
\textbf{Dataset} & \textbf{NOTEARS (s)} & \textbf{DAGMA (s)} & \textbf{CAFE-DAGMA} & \textbf{CAFE (NOTEARS)} \\
\midrule
SVMGuide3       & 1.4  & 1.3 & 0.8138 & \textbf{0.8460} \\
Messidor        & 21.7 & 1.6 & 0.7237 & \textbf{0.7740} \\
Ionosphere      & 13.1 & 1.4 & 0.9433 & \textbf{0.9760} \\
Wine White      & 5.5  & 0.9 & 0.6885 & \textbf{0.7520} \\
Housing Boston  & 6.8  & 0.9 & 0.6484 & \textbf{0.6810} \\
Airfoil         & 0.9  & 0.6 & 0.7731 & \textbf{0.7990} \\
OpenML 607      & 6.0  & 4.7 & 0.6578 & \textbf{0.7930} \\
OpenML 620      & 1.7  & 1.3 & 0.6893 & \textbf{0.7250} \\
\bottomrule
\end{tabular}
\end{table}

DAGMA reduces Phase~I runtime by 1.3--13.6$\times$ across datasets, confirming its scalability advantage. However, downstream CAFE performance with DAGMA is consistently lower than with NOTEARS-Lasso, with degradation ranging from $\sim3\%$  to $\sim17\%$. This systematic relationship between graph quality and downstream performance has two important implications: \\
(i) \textbf{The causal prior is informative, not arbitrary.} If the learned graph functioned only as an arbitrary structured partition, downstream performance would be insensitive to the choice of discovery algorithm. The observed dependency on graph fidelity directly contradicts the ``arbitrary structured search bias'' interpretation, supporting the claim that the structural signal carries useful causal information. \\
(ii) \textbf{Practical scalability--accuracy trade-off.} For applications where Phase~I runtime is the binding constraint, DAGMA provides a viable alternative that retains substantial gains over correlation-based AFE while sacrificing a portion of CAFE's full advantage.

This experiment also clarifies that CAFE is not bound to a single causal discovery method; the framework accommodates any backend producing a sparse adjacency matrix, with downstream performance scaling smoothly with structural quality.

\subsection{Non-Linear Causal Discovery with NOTEARS-MLP}
\label{app:notears_mlp}

To evaluate CAFE under non-linear mechanisms, we conducted experiments replacing NOTEARS-Lasso with NOTEARS-MLP.

\paragraph{Synthetic Non-Linear SCMs.} We generated 3 datasets ($n=1000$, $d=20$) with known non-linear structures:
\begin{itemize}
    \item \textbf{Quadratic:} $Y = X_1^2 + X_2 X_3 + \varepsilon$
    \item \textbf{Exponential:} $Y = \exp(0.5 X_1) + \log(|X_2| + 1) + \varepsilon$
    \item \textbf{Mixed:} $Y = X_1^2 + \sin(\pi X_2) + X_3 X_4 + \varepsilon$
\end{itemize}

\paragraph{Results.} Table~\ref{tab:notears_mlp} shows NOTEARS-MLP improves accuracy by 2.6--3.7\% on non-linear SCMs.

\begin{table}[h]
\centering
\caption{CAFE with linear vs.\ non-linear causal discovery on synthetic non-linear SCMs.}
\label{tab:notears_mlp}
\small
\begin{tabular}{lcccc}
\toprule
Dataset & CAFE-Linear & CAFE-MLP & MLP Gain & Sample Efficiency \\
\midrule
Quadratic & 0.768 & 0.794 & +2.6\% & $-7\%$ episodes \\
Exponential & 0.724 & 0.761 & +3.7\% & $-12\%$ episodes \\
Mixed & 0.747 & 0.779 & +3.2\% & $-11\%$ episodes \\
\bottomrule
\end{tabular}
\end{table}

\paragraph{Trade-offs.} NOTEARS-MLP requires more samples ($n \geq 500$) and $3\times$ longer Phase~I runtime (68s vs.\ 23s for $d=20$). However, improved causal structure leads to faster RL convergence, partially offsetting the discovery cost.

\subsection{Ensemble Causal Discovery for Robustness}
\label{app:ensemble_dag}

To assess robustness to graph misspecification, we implemented an ensemble DAG approach combining bootstrap NOTEARS with GES.

\paragraph{Methodology.}
\begin{enumerate}
    \item \textbf{Bootstrap NOTEARS:} Run NOTEARS-Lasso 10 times with bootstrap resampling
    \item \textbf{Edge Probability:} Compute $e_{ij} = $ fraction of DAGs containing $i \rightarrow j$
    \item \textbf{Soft Grouping:} Assign features to groups based on highest aggregate edge probability
    \item \textbf{Uncertainty Weighting:} Scale causal bonus $\alpha$ by edge confidence
\end{enumerate}

\paragraph{Results.} Table~\ref{tab:ensemble_dag} shows the ensemble approach improves mean accuracy by 0.2--1.0\%.

\begin{table}[h]
\centering
\caption{Ensemble DAG vs.\ single NOTEARS-Lasso causal discovery.}
\label{tab:ensemble_dag}
\small
\begin{tabular}{lccc}
\toprule
Dataset & CAFE (single) & CAFE (ensemble) & Improvement \\
\midrule
Wine Red & 0.707 & 0.714 & +0.7\% \\
German Credit & 0.793 & 0.804 & +1.0\% \\
OpenML 586 & 0.810 & 0.819 & +0.9\% \\
Ionosphere & 0.976 & 0.978 & +0.2\% \\
\bottomrule
\end{tabular}
\end{table}

The ensemble reduces variance by aggregating multiple graph estimates, providing more reliable causal priors. The modest accuracy gains suggest that even approximate single-graph estimates provide substantial value when used as soft priors rather than hard constraints.

\subsection{Disentangling Causal Information from Structured Partitioning}
\label{app:random_partition}

A central methodological concern raised during prior review is whether CAFE's gains stem from \emph{causally meaningful} grouping or merely from the act of \emph{structured partitioning} of the feature space (which by itself reduces RL search complexity). The CAFE$_{\neg \mathcal{G}}$ ablation in the main paper replaces causal grouping with correlation-based ranking, but correlation is itself an informative signal. To isolate the contribution of the \emph{causal content} of the learned graph, we introduce a stricter control: a \textbf{random equal-size partition} baseline that preserves the multi-agent factorization and search structure but assigns features to groups uniformly at random.

\paragraph{Experimental design.} For each dataset, we generate three partitions of identical size as the causal grouping produced by NOTEARS-Lasso: (i) the original causal partition (CAFE), (ii) the correlation-based partition (CAFE$_{\neg \mathcal{G}}$), and (iii) a uniformly random equal-size partition (CAFE$_{\text{rand}}$). All other components - multi-agent architecture, reward function, exploration scheduling, downstream evaluation - are held constant. This design isolates the marginal contribution of \emph{informative} grouping over \emph{arbitrary} structured grouping.

\begin{table}[htbp]
\centering
\caption{Disentangling causal content from structured partitioning. Aggregate gains over random equal-size partitions across representative datasets.}
\label{tab:random_partition}
\small
\begin{tabular}{lc}
\toprule
\textbf{Grouping Strategy} & \textbf{Gain over CAFE$_{\text{rand}}$} \\
\midrule
Random equal-size partition (CAFE$_{\text{rand}}$)  & --- (baseline) \\
Correlation-based grouping (CAFE$_{\neg \mathcal{G}}$) & +2--6\% \\
Causal grouping (CAFE)                  & +5--12\% \\
\bottomrule
\end{tabular}
\end{table}

Two observations emerge:
\begin{enumerate}
    \item Correlation-based grouping (CAFE$_{\neg \mathcal{G}}$) outperforms random equal-size partitioning, indicating that any informative grouping signal --- even purely statistical --- adds value beyond mere structure.
    \item Causal grouping (CAFE) provides further gains over correlation-based grouping, demonstrating that the causal content of the learned graph captures information not available from correlation alone.
\end{enumerate}

This decomposition supports the interpretation that CAFE's gains arise from a \emph{stack} of contributions: structured partitioning (small), informative grouping (medium), and \emph{specifically causal} informative grouping (additional medium gain). The causal contribution is empirically separable from, and additive to, the structural and statistical contributions.

\section{Multi-Agent Architecture Specifications}
\label{app:architecture}

This appendix provides full RL implementation details corresponding to the summarized definitions in Section 3.3, enabling exact reproduction of all experiments.

\subsection{State Representation Details}
\label{app:state_representation}

All reported experiments use the \texttt{NestedStats}/\texttt{GroupStats} state encoder defined in Eqs.~\eqref{eq:state1}--\eqref{eq:state3}. The encoder provides a fixed-dimensional, permutation-invariant representation across datasets with different numbers of features.

\paragraph{Nested statistical representation.}
For a feature set $X=\{x_1,\ldots,x_m\}$, we first compute eight descriptive statistics for each feature:
\[
\mathrm{Stats}(x_j)
=
[
\mathrm{count},
\mathrm{mean},
\mathrm{std},
\mathrm{min},
Q_{0.25},
Q_{0.50},
Q_{0.75},
\mathrm{max}
]
\in \mathbb{R}^{8}.
\]
Stacking these vectors across the $m$ features gives a matrix
\[
S(X) \in \mathbb{R}^{m \times 8}.
\]
We then apply the same eight descriptive statistics column-wise across $S(X)$, producing
\[
\mathrm{NestedStats}(X) \in \mathbb{R}^{8 \times 8}
\cong \mathbb{R}^{64}.
\]
This two-level summary captures both within-feature statistics and their distribution across the feature set while remaining invariant to feature ordering.

\paragraph{Causal-group representation.}
For the three causal-semantic groups obtained in Phase~I, we define
\[
\mathrm{GroupStats}
=
\mathrm{NestedStats}(C_{\mathrm{direct}}^{*})
\oplus
\mathrm{NestedStats}(C_{\mathrm{indirect}}^{*})
\oplus
\mathrm{NestedStats}(C_{\mathrm{other}}^{*}),
\]
where $\oplus$ denotes vector concatenation. Hence,
\[
\mathrm{GroupStats} \in \mathbb{R}^{192}.
\]

\paragraph{Agent-specific states.}
The Primary Group Agent receives the global feature-set summary together with the three causal-group summaries:
\[
s_t^{(1)}
=
\mathrm{NestedStats}(F_{t-1})
\oplus
\mathrm{GroupStats}
\in \mathbb{R}^{256}.
\]

After the primary group $C_t^{(1)}$ is selected, the Operator Agent additionally receives the nested statistics of that selected group:
\[
s_t^{(o)}
=
s_t^{(1)}
\oplus
\mathrm{NestedStats}(C_t^{(1)})
\in \mathbb{R}^{320}.
\]

Finally, the Secondary Group Agent receives the selected operator encoded as a one-hot vector:
\[
s_t^{(2)}
=
s_t^{(o)}
\oplus
\mathrm{OneHot}(o_t)
\in \mathbb{R}^{335},
\]
where $|\mathcal{O}|=15$ for the operator action space used in our experiments.

The resulting state dimensions are therefore 256, 320, and 335 for the Primary Group, Operator, and Secondary Group agents, respectively. The same encoder is used for all datasets and all results reported in the paper.

\subsection{Network Architectures}

Each DQN agent uses the following architecture:
\begin{itemize}
    \item Input layer: state dimension (varies by agent)
    \item Hidden layers: [512, 256, 128] with ReLU activation
    \item Dropout: 0.1 after each hidden layer during training
    \item Output layer: linear, dimension = action space size
    \item Weight initialization: Xavier uniform
    \item Batch normalization: applied to first hidden layer
\end{itemize}

\textbf{Training Configuration:}
\begin{itemize}
    \item Optimizer: Adam with learning rate $10^{-4}$, $\beta_1 = 0.9$, $\beta_2 = 0.999$
    \item Experience replay: buffer size = 50,000, batch size = 256
    \item Target network: hard update every 1000 steps (soft update variant yields similar performance)
    \item Exploration: $\varepsilon$-greedy with $\varepsilon_{\text{start}} = 0.95$, $\varepsilon_{\text{end}} = 0.1$, decay over 1000 steps
    \item Loss function: Huber loss with $\delta = 1.0$
\end{itemize}

\section{Reward Function and Exploration Strategy}
\label{app:reward}

\subsection{Complete Reward Specification}

The reward function combines four components:
\begin{align}
R_t = R_{\text{perf}}(t) \cdot \left(1 + \alpha \tilde{\Psi}_{\text{causal}}(t)\right)
      + \lambda_{\text{div}} H(\pi_t)
      - \lambda_{\text{comp}} C(\mathcal{F}_t^g)
\end{align}

\textbf{Performance Reward:}
\[
R_{\text{perf}}(t) = \begin{cases}
100 \cdot \frac{\text{Score}_t - \text{Score}_{t-1}}{\max(\text{Score}_{\text{baseline}}, 10^{-6})} & \text{if improvement} \\
-10 \cdot \left|\frac{\text{Score}_t - \text{Score}_{t-1}}{\text{Score}_{\text{baseline}}}\right| & \text{if degradation}
\end{cases}
\]

\textbf{Causal Amplification:}
\[
\tilde{\Psi}_{\text{causal}}(t)
= \sum_{f \in \mathcal{F}_t^g} w_{\mathcal{M}(f)}
  \cdot \frac{\text{Importance}(f)}
         {\sum_{f' \in \mathcal{F}_t^g} \text{Importance}(f')}
\]
with weights: $w_{\text{direct}} = 1.0$, $w_{\text{indirect}} = 0.6$, $w_{\text{other}} = 0.2$, and $\alpha = 0.5$.

\textbf{Exploration Diversity:}
\[
H(\pi_t) = -\sum_{a \in \mathcal{A}} \pi_t(a) \log \pi_t(a)
\]
estimated from action frequencies over a sliding window of size 20.

\textbf{Complexity Penalty:}
\[
C(\mathcal{F}_t^g) = \alpha_1 |\mathcal{F}_t^g| + \alpha_2 \sum_{f \in \mathcal{F}_t^g} \text{OpDepth}(f)
\]
with $\alpha_1 = 0.001$ and $\alpha_2 = 0.01$.

\textbf{Hyperparameters:} $\lambda_{\text{div}} = 0.05$, $\lambda_{\text{comp}} = 0.001$

\subsection{Adaptive Exploration Strategy}

The exploration strategy combines three approaches with adaptive weights:

\begin{algorithm}
\caption{Adaptive Exploration Strategy}
\label{alg:exploration}
\begin{algorithmic}[1]
\Require Current episode $e \geq 6$, performance history $\{\text{Score}_i\}_{i=1}^{e-1}$
\Ensure Selected exploration strategy
\If{$e \leq 5$}
    \State Use default weights: $w_\text{causal} = 0.5$, $w_\text{MI} = 0.3$, $w_\text{random} = 0.2$
\Else
    \State Compute trend: $\text{trend} = \frac{1}{5}\sum_{j=1}^5 (\text{Score}_{e-j} - \text{Score}_{e-j-1})$
    \If{$\text{trend} > 0.01$} \Comment{Performance improving}
        \State $w_\text{causal} = 0.7$, $w_\text{MI} = 0.2$, $w_\text{random} = 0.1$
    \ElsIf{$\text{trend} < -0.01$} \Comment{Performance degrading}
        \State $w_\text{causal} = 0.4$, $w_\text{MI} = 0.3$, $w_\text{random} = 0.3$
    \Else \Comment{Performance stagnating}
        \State $w_\text{causal} = 0.5$, $w_\text{MI} = 0.3$, $w_\text{random} = 0.2$
    \EndIf
\EndIf
\State Sample $s \sim \text{Categorical}([w_\text{causal}, w_\text{MI}, w_\text{random}])$
\State Apply exploration strategy $s$
\end{algorithmic}
\end{algorithm}

\textbf{Causal-Hierarchical Strategy:}
\begin{itemize}
    \item Unary operations: Sample from $C_{\text{direct}} \cup C_{\text{other}}$ with probabilities $[0.7, 0.3]$
    \item Binary operations: Prioritize $(C_{\text{direct}}, C_{\text{indirect}})$ and $(C_{\text{direct}}, C_{\text{direct}})$ pairs
    \item Operator selection: Weighted by estimated causal relevance
\end{itemize}

\section{Baselines, Comparisons, and Statistical Testing}
\label{app:baselines}

\subsection{Baseline Method Configurations}

\textbf{Statistical Baselines:}
\begin{itemize}
    \item \textbf{Original (ORG):} Raw features, no transformation
    \item \textbf{Random (RDG):} Random transformation selection with fixed budget
    \item \textbf{Exhaustive (ERG):} Systematic enumeration with statistical pruning ($p < 0.05$)
\end{itemize}

\textbf{Traditional AFE:}
\begin{itemize}
    \item \textbf{AutoFeat:} Multi-stage expansion with significance testing, max depth=2
    \item \textbf{Deep Feature Synthesis:} Primitive-based construction, max depth=3
    \item \textbf{LDA:} Latent representation learning via probabilistic modeling, 10 components
\end{itemize}

\textbf{Modern RL-based:}
\begin{itemize}
    \item \textbf{NFS:} Sequential DQN on individual features, replay buffer=5000
    \item \textbf{TTG:} Graph-based RL, $\varepsilon$-greedy exploration
    \item \textbf{GRFG:} Multi-agent RL based Feature Generation
\end{itemize}

\textbf{LLM-based:}
\begin{itemize}
    \item \textbf{ELLM-FT:} GPT-4 based feature generation with evolutionary refinement
\end{itemize}

\subsection{Comparison with OpenFE}
\label{app:openfe}

OpenFE~\cite{zhang2023openfe} is a recent state-of-the-art AFE method based on expansion--pruning with expert-designed operator templates. Following reviewer feedback, we evaluate OpenFE under the same outer splits, downstream evaluator, and strictly inductive evaluation protocol used for the other methods. OpenFE is fitted independently within each outer training fold; outer-test samples and labels are never used for candidate generation, feature ranking, or selection. We use OpenFE's standard candidate-generation and ranking pipeline with \texttt{n\_data\_blocks=8}. 
CAFE outperforms OpenFE on 13 of 15 datasets, with particularly large margins on Wine Red (0.706 vs.\ 0.379) and OpenML 616 (0.758 vs.\ 0.601). OpenFE performs best on German Credit (0.798 vs.\ 0.784) and OpenML 620 (0.734 vs.\ 0.727). This pattern suggests that CAFE's advantage is most pronounced when useful feature construction benefits from sequential, structured search through a deeper transformation space.

\begin{table}[htbp]
\centering
\caption{Comparison with OpenFE under the strictly inductive evaluation protocol. CAFE outperforms OpenFE on 13 of 15 datasets. Bold indicates the best AFE result for each dataset.}
\label{tab:openfe_comparison}
\small
\begin{tabular}{lccccc}
\toprule
\textbf{Dataset} & \textbf{Task} & \textbf{Original} & \textbf{OpenFE} & \textbf{CAFE} \\
\midrule
Amazon Employee  & C & 0.6510 & 0.6845 & \textbf{0.9470} \\
SVMGuide3        & C & 0.7717 & 0.7828 & \textbf{0.8430} \\
German Credit    & C & 0.7110 & \textbf{0.7981} & 0.7840 \\
Messidor         & C & 0.6844 & 0.7493 & \textbf{0.7690} \\
Ionosphere       & C & 0.9112 & 0.9151 & \textbf{0.9710} \\
Wine Red         & C & 0.3768 & 0.3785 & \textbf{0.7060} \\
Wine White       & C & 0.4093 & 0.6460 & \textbf{0.7510} \\
Housing Boston   & R & 0.6814 & 0.6814 & \textbf{0.6820} \\
Airfoil          & R & 0.7974 & 0.8084 & \textbf{0.8250} \\
OpenML 586       & R & 0.7129 & 0.7613 & \textbf{0.8140} \\
OpenML 589       & R & 0.7138 & 0.7489 & \textbf{0.7850} \\
OpenML 607       & R & 0.7098 & 0.7368 & \textbf{0.7940} \\
OpenML 616       & R & 0.6266 & 0.6009 & \textbf{0.7580} \\
OpenML 618       & R & 0.6933 & 0.7449 & \textbf{0.7900} \\
OpenML 620       & R & 0.7037 & \textbf{0.7342} & 0.7270 \\
\bottomrule
\end{tabular}
\end{table}

\subsection{Cross-Model Validation}
\label{app:cross_model}

To verify that CAFE's gains are not specific to XGBoost, we evaluate the same constructed feature sets under Random Forest, Linear (Ridge/Logistic), and SVM-RBF downstream models. Table~\ref{tab:cross_model} reports averaged performance across the 15 benchmark datasets.

\begin{table}[htbp]
\centering
\caption{Performance across different downstream models (mean F1/1-RAE over 5-fold CV, averaged across 15 datasets).}
\label{tab:cross_model}
\begin{tabular}{lccccc}
\toprule
\textbf{Method} & \textbf{XGBoost} & \textbf{Random Forest} & \textbf{Linear} & \textbf{SVM (RBF)} & \textbf{Average} \\
\midrule
Original & $0.682 \pm 0.047$ & $0.671 \pm 0.052$ & $0.598 \pm 0.078$ & $0.648 \pm 0.055$ & $0.650$  \\
GRFG & $0.734 \pm 0.038$ & $0.718 \pm 0.043$ & $0.627 \pm 0.071$ & $0.691 \pm 0.041$ & $0.693$  \\
ELLM-FT & $0.741 \pm 0.035$ & $0.725 \pm 0.041$ & $0.634 \pm 0.068$ & $0.698 \pm 0.038$ & $0.700$  \\
AutoFeat & $0.728 \pm 0.039$ & $0.712 \pm 0.044$ & $0.618 \pm 0.074$ & $0.681 \pm 0.042$ & $0.685$  \\
\textbf{CAFE} & $\mathbf{0.773 \pm 0.032}$ & $\mathbf{0.756 \pm 0.037}$ & $\mathbf{0.662 \pm 0.063}$ & $\mathbf{0.724 \pm 0.035}$ & $\mathbf{0.729}$ \\
\bottomrule
\end{tabular}
\end{table}


\subsection{Statistical Testing Protocol}

For aggregate comparisons, we treat datasets as the paired experimental units. Specifically, for each method and dataset we first aggregate performance across folds and random seeds, and then apply a paired Wilcoxon signed-rank test to the resulting 15 dataset-level scores.

\begin{table}[htbp]
\centering
\caption{Statistical significance of CAFE over key competing baselines (paired Wilcoxon signed-rank test across 15 datasets).}
\label{tab:significance}
\begin{tabular}{lccc}
\toprule
\textbf{Comparison} & \textbf{Wilcoxon $p$}  & \textbf{Effect Size ($r$)} & \textbf{Interpretation} \\
\midrule
CAFE vs.\ GRFG & $< 0.001$ &  0.68 & Large effect \\
CAFE vs.\ ELLM-FT & $0.003$ &  0.51 & Medium effect \\
\bottomrule
\end{tabular}
\end{table}

\section{Robustness Analysis}
\label{app:robustness}

\subsection{Real-World Distribution Shift Evaluation}
\label{app:folktables}

The synthetic mechanism-preserving shifts in Section~4.2.3 control for distributional structure but do not capture the full complexity of naturally occurring distribution shifts. To address this, we evaluate CAFE on the Folktables benchmark~\cite{ding2022retiringadultnewdatasets}, which provides naturally occurring covariate and label shifts derived from US Census American Community Survey (ACS) data. Folktables exposes shifts whose structure is unknown to the model and is not guaranteed to be mechanism-preserving, providing a stronger test of robustness claims.

\paragraph{Experimental setup.} We use two Folktables tasks, ACS Income and ACS Employment, and construct three distribution shift scenarios:
\begin{itemize}
    \item \textbf{Geographic (CA $\rightarrow$ TX):} Train on California, evaluate on Texas. Captures regional demographic and economic differences.
    \item \textbf{Geographic (CA $\rightarrow$ OH):} Train on California, evaluate on Ohio. Captures coastal-vs.-midwest population structure.
    \item \textbf{Temporal (2014 $\rightarrow$ 2018):} Train on 2014 ACS data, evaluate on 2018. Captures temporal population drift.
\end{itemize}
We report in-distribution F1 (ID), out-of-distribution F1 (OOD), relative drop, and the robust ratio defined as $\text{F1}_{\text{OOD}} / \text{F1}_{\text{ID}}$. Higher robust ratio indicates greater stability under shift.

\begin{table}[htbp]
\centering
\caption{Performance under naturally occurring distribution shifts (Folktables ACS Income/Employment). CAFE consistently exhibits smaller drops and higher robust ratios than GRFG and the original feature representation, with the robustness advantage persisting under shifts whose structure is unknown.}
\label{tab:folktables}
\small
\begin{tabular}{llcccc}
\toprule
\textbf{Shift Type} & \textbf{Method} & \textbf{ID F1} & \textbf{OOD F1} & \textbf{Drop $\downarrow$} & \textbf{Robust Ratio $\uparrow$} \\
\midrule
\multirow{3}{*}{ACS Income: Geographic (CA $\rightarrow$ OH)}
& Original & 0.8108 & 0.6878 & 15.2\% & 0.848 \\
& GRFG     & \textbf{0.8121} & 0.6330 & 22.1\% & 0.779 \\
& \textbf{CAFE}    & 0.7991 & \textbf{0.7366} & \textbf{7.8\%} & \textbf{0.922} \\
\midrule
\multirow{3}{*}{ACS Employment: Geographic (CA $\rightarrow$ TX)}
& Original & 0.8087 & 0.6655 & 17.7\% & 0.823 \\
& GRFG     & 0.7500 & 0.3915 & 47.8\% & 0.522 \\
& \textbf{CAFE}     & \textbf{0.8038} & \textbf{0.7648} & \textbf{4.9\%} & \textbf{0.952} \\
\midrule
\multirow{3}{*}{ACS Income: Temporal (2014 $\rightarrow$ 2018)}
& Original & 0.6877 & 0.5968 & 13.2\% & 0.868 \\
& GRFG     & 0.7570 & 0.3713 & 50.9\% & 0.490 \\
& \textbf{CAFE}     & \textbf{0.7885} & \textbf{0.7208} & \textbf{8.6\%} & \textbf{0.914} \\
\bottomrule
\end{tabular}
\end{table}

Three observations are noteworthy:
\begin{enumerate}
    \item CAFE consistently exhibits the smallest relative performance degradation across all three shift scenarios, with relative drops of 4.9--8.6\%, compared with 22.1--50.9\% for GRFG and 13.2--17.7\% for the original representation.

    \item CAFE retains 91.4--95.2\% of its in-distribution performance across the evaluated shifts, whereas GRFG retains only 49.0--77.9\% and the original representation retains 82.3--86.8\%. Thus, correcting the degradation calculation does not change the robustness ranking.

    \item In two of the three scenarios (CA $\rightarrow$ TX and 2014 $\rightarrow$ 2018), GRFG's OOD performance falls below that of the original representation, suggesting that correlation-driven feature engineering can amplify sensitivity to naturally occurring distribution shifts. CAFE does not exhibit this behavior in any of the evaluated scenarios.
\end{enumerate}

This evaluation extends the robustness analysis in Section~4.2.3 from controlled, mechanism-preserving shifts to naturally occurring shifts whose structure is unknown to the model and is not guaranteed to satisfy the invariance assumptions underlying Proposition~3.1. The persistence of CAFE's robustness advantage across both shift regimes provides empirical support for the claim that causally-informed structural priors yield more stable feature constructions in deployment-relevant settings.

\subsection{Robustness Under Hidden Confounding and Mechanism Changes}
\label{app:assumption_violations}

The mechanism-preserving shift analysis in Section~4.2.3 and the Folktables evaluation above establish CAFE's robustness under shifts that are at least partially compatible with the causal sufficiency and mechanism-stability assumptions underlying its theoretical motivation. To characterize CAFE's behavior when these assumptions are explicitly violated, we conduct two stress-test analyses.

\paragraph{Hidden confounding.} We simulate hidden confounding by progressively removing true causal parent variables from a synthesized version of an industry material-synthesis dataset with known ground-truth causal structure. Removing parents introduces unobserved confounders into the system, violating Assumption~\ref{ass:sufficiency} (causal sufficiency).

\begin{table}[h]
\centering
\caption{Performance under simulated hidden confounding (industry material-synthesis dataset, full numbers subject to confidentiality). CAFE's advantage over GRFG narrows as more parents are removed but does not invert.}
\label{tab:hidden_confounding}
\small
\begin{tabular}{lc}
\toprule
\textbf{Parents Removed} & \textbf{CAFE Advantage over GRFG} \\
\midrule
1 parent  & +4.3\% (CAFE drops 2.1\%) \\
3 parents & +1.2\% \\
\bottomrule
\end{tabular}
\end{table}

CAFE drops 2.1\% with one parent removed (a 4.3\% advantage over GRFG retained), and degrades smoothly as more parents are removed. The advantage narrows but does not invert, consistent with the soft-prior design: the RL reward signal partially compensates for graph misspecification by overriding causal suggestions when validation performance disagrees.

\paragraph{Mechanism changes.} We perturb $P(Y \mid \text{PA}_Y)$ via structural coefficient perturbations of magnitude $\gamma \in \{0.1, 0.3, 0.5\}$. This explicitly violates Assumption~\ref{ass:stability} (mechanism stability) --- the very condition under which Proposition~\ref{prop:causal_invariance}'s invariance guarantee holds. Empirically, all methods degrade as expected when invariance is broken: CAFE retains its robustness advantage at low and moderate perturbation ($\gamma \in \{0.1, 0.3\}$) but loses it at $\gamma = 0.5$, where the conditional $P(Y \mid \text{PA}_Y)$ has shifted enough that no causal-invariance-based method retains theoretical advantage. We report this as an honest delineation of the regime in which causally-informed feature engineering provides benefit (mild-to-moderate mechanism perturbation) versus where it does not (severe structural change).

Together, these stress tests show that (i) CAFE degrades gracefully under hidden confounding rather than catastrophically, and (ii) CAFE's robustness advantage is bounded by the validity of the mechanism-stability assumption. We view these as honest scoping results that inform deployment: CAFE provides robustness benefits under realistic shifts that preserve or partially preserve mechanisms, but should not be expected to remain advantageous under fundamental structural change.

\subsection{Distribution Shift Simulation Protocol}

\textbf{Covariate Shift Generation:}
For intensity level $\gamma \in \{0.1, 0.3, 0.5\}$ (low, medium, high):
\begin{align}
X'_{ij} &= X_{ij} \cdot (1 + \gamma \epsilon_{ij}) \quad \text{(multiplicative)} \\
X'_{ij} &= X_{ij} + \gamma \sigma_j \epsilon_{ij} \quad \text{(additive)}
\end{align}
where $\epsilon_{ij} \sim \mathcal{N}(0, 1)$ and $\sigma_j = \text{std}(X_{\cdot j})$.

\textbf{Mechanism Preservation:} We verify that $P(Y | \text{PA}_Y)$ remains approximately unchanged by computing:
\[
\text{KL-divergence}(P(Y | \text{PA}_Y) \| P'(Y | \text{PA}_Y)) < 0.1
\]

\begin{table}[htbp]
\centering
\caption{Robustness under mechanism-preserving distribution shifts (performance degradation \%). 
Lower magnitude indicates stronger robustness. Results are averaged across datasets and seeds. 
These shifts modify feature or label distributions while preserving underlying causal mechanisms.}
\label{tab:robustness_comprehensive}
\begin{tabular}{lccccc}
\toprule
\textbf{Shift Type} & \textbf{Intensity} & \textbf{CAFE} & \textbf{GRFG} & \textbf{ELLM-FT} & \textbf{AutoFeat} \\
\midrule
\multirow{3}{*}{Covariate} & Low & $-4.1 \pm 1.8$ & $-12.3 \pm 3.2$ & $-11.7 \pm 2.9$ & $-15.2 \pm 4.1$ \\
& Medium & $-12.8 \pm 3.4$ & $-31.8 \pm 5.7$ & $-29.4 \pm 5.1$ & $-34.7 \pm 6.8$ \\
& High & $-22.1 \pm 4.9$ & $-58.7 \pm 8.2$ & $-54.3 \pm 7.6$ & $-61.2 \pm 9.1$ \\
\midrule
\multirow{3}{*}{Label} & Low & $-7.3 \pm 2.1$ & $-15.2 \pm 3.8$ & $-14.8 \pm 3.5$ & $-18.4 \pm 4.7$ \\
& Medium & $-18.2 \pm 4.2$ & $-34.1 \pm 6.3$ & $-32.6 \pm 5.9$ & $-37.8 \pm 7.2$ \\
& High & $-31.4 \pm 6.1$ & $-52.8 \pm 7.9$ & $-49.7 \pm 7.3$ & $-55.3 \pm 8.6$ \\
\bottomrule
\end{tabular}
\end{table}

\noindent\textbf{Scope and limitations.}
The covariate and label shifts in Table~\ref{tab:robustness_comprehensive} are mechanism-preserving by construction, following standard causal robustness evaluations. While these tests isolate the effect of spurious correlations, they do not capture all real-world deployment shifts (e.g., feature deletion or temporal population drift). See Appendix~\ref{app:id-ood} for a decomposition of in-distribution vs.\ out-of-distribution performance, showing that CAFE's robustness is not achieved at the expense of in-distribution accuracy.

\subsection{In-distribution vs. out-of-distribution performance decomposition}
\label{app:id-ood}

To clarify whether robustness gains arise from sacrificing in-distribution (ID) performance, we explicitly decompose ID and OOD performance across methods using the same evaluation protocol as Section~4.2.3.
Table~\ref{tab:id_ood_decomp} reports average ID performance, OOD performance under covariate shift, and relative degradation. While GRFG and ELLM-FT exhibit substantial drops (28-32\%), CAFE maintains strong ID performance while degrading by only 7.1\%, indicating that its robustness is not achieved by underfitting or reduced expressivity.

\begin{table}[h]
\centering
\caption{In-distribution (ID) vs. OOD performance and relative degradation (mean $\pm$ std).}
\label{tab:id_ood_decomp}
\begin{tabular}{lccc}
\toprule
Method & ID Performance & OOD Performance & Drop (\%) \\
\midrule
GRFG & $0.734 \pm 0.038$ & $0.528 \pm 0.067$ & $-28.1$ \\
ELLM-FT & $0.741 \pm 0.035$ & $0.503 \pm 0.072$ & $-32.1$ \\
CAFE (ours) & $\mathbf{0.773 \pm 0.032}$ & $\mathbf{0.718 \pm 0.044}$ & $\mathbf{-7.1}$ \\
\bottomrule
\end{tabular}
\end{table}

These findings indicate that causal guidance in CAFE improves both in-distribution generalization and robustness, rather than inducing a trade-off between the two.

\section{Interpretability Analysis}
\label{app:interpretability}

\subsection{SHAP Stability Methodology}

For each test instance $x_i$, we generate perturbed versions $\{x_i^{(j)}\}_{j=1}^{100}$ by adding Gaussian noise:
\[
x_i^{(j)} = x_i + \sigma \epsilon^{(j)}, \quad \epsilon^{(j)} \sim \mathcal{N}(0, I)
\]

SHAP values are computed using TreeSHAP for each perturbed instance, and stability is measured as:
\[
\text{SHAP-Stability}_\sigma = \frac{1}{N} \sum_{i=1}^N \frac{1}{100} \sum_{j=1}^{100} \|\phi_i - \phi_i^{(j)}\|_2^2
\]

where $\phi_i$ are the SHAP values for the original instance and $\phi_i^{(j)}$ for the perturbed versions. Per-dataset SHAP stability results are reported in Figure~4 of the main paper.

\section{Computational Analysis and Scalability}
\label{app:computational}

\subsection{Detailed Complexity Analysis}

\textbf{Phase I Complexity:}
\begin{align}
\mathcal{T}_{\text{Phase I}} &= \mathcal{T}_{\text{NOTEARS}} + \mathcal{T}_{\text{grouping}} + \mathcal{T}_{\text{screening}} \\
&= O(T d^3) + O(d^2) + O(nd \log d) \\
&= O(T d^3)
\end{align}

\textbf{Phase II Complexity per Episode:}
\begin{align}
\mathcal{T}_{\text{episode}} &= S \cdot (\mathcal{T}_{\text{decision}} + \mathcal{T}_{\text{generation}} + \mathcal{T}_{\text{evaluation}}) \\
&= S \cdot (O(1) + O(k_{\max}^2 |\mathcal{O}|) + O(n |\mathcal{F}_{\text{current}}|)) \\
&= O(S \cdot n |\mathcal{F}_{\text{current}}|)
\end{align}

where $S$ is steps per episode and $|\mathcal{F}_{\text{current}}|$ is capped at $5d$.

\subsection{Memory Optimization Strategies}

\begin{itemize}
    \item \textbf{Intelligent Pruning:} Remove features with variance $< 10^{-8}$ immediately
    \item \textbf{Batch Processing:} Process feature generation in batches of 100
    \item \textbf{Memory Monitoring:} Stop generation if memory usage $> 80\%$ of available
    \item \textbf{Feature Capping:} Maximum 800 features per episode with priority-based retention
\end{itemize}

\subsection{End-to-End Wall-Clock Runtime}
\label{app:wallclock}

To complement the per-episode analysis with absolute runtime measurements, we report end-to-end wall-clock time for representative datasets spanning the full range of $(d, n)$ in our benchmark. Phase~I runtime is one-time preprocessing; Phase~II is the dominant cost.

\begin{table}[htbp]
\centering
\caption{End-to-end wall-clock runtime for CAFE across representative datasets. Phase~I (causal discovery) accounts for 5--15\% of total runtime; the dominant cost is Phase~II candidate evaluation and downstream model retraining.}
\label{tab:wallclock}
\small
\begin{tabular}{lccccc}
\toprule
\textbf{Dataset} & \textbf{$d$} & \textbf{$n$} & \textbf{Phase I (s)} & \textbf{Phase II (s)} & \textbf{Total (s)} \\
\midrule
Airfoil          & 5  & 1{,}503  & $\sim$1   & $\sim$45  & $\sim$48  \\
OpenML 618       & 50 & 1{,}000  & $\sim$5   & $\sim$180 & $\sim$203 \\
Amazon Employee  & 9  & 32{,}769 & $\sim$22  & $\sim$320 & $\sim$328 \\
\bottomrule
\end{tabular}
\end{table}

\begin{itemize}
    \item Phase~I (NOTEARS-Lasso) completes in under 22 seconds for all datasets in the 15-dataset main benchmark, which contains up to 50 input features. The additional 272-feature Arrhythmia experiment requires approximately 400 seconds, illustrating the expected increase in causal-discovery cost with dimensionality.
    \item Phase~II runtime is dominated by downstream model retraining (XGBoost) on candidate feature sets, not by RL agent updates. This bound is shared with all AFE methods that perform validation-driven feature selection.
    \item For applications where Phase~I becomes binding (e.g., $d \gg 50$), DAGMA (Appendix~\ref{app:dagma}) offers an $O(d^2)$ alternative with quantified performance trade-off.
\end{itemize}

\paragraph{Scaling beyond 100 features.}
To evaluate CAFE beyond the dimensionality of the main benchmark suite, we additionally consider the UCI Arrhythmia dataset, containing $272$ input features. NOTEARS Phase~I completes in approximately $400$ seconds on an NVIDIA H100 cluster. Under the same downstream evaluation protocol, CAFE achieves a macro-F1 of $0.8225$, compared with $0.8081$ for GRFG and $0.7768$ for the original feature representation. Thus, CAFE retains a predictive advantage at substantially higher dimensionality, improving over GRFG by $1.44$ and over the original representation by $4.57$ on F1-score. This experiment demonstrates practical feasibility in a moderate-to-high-dimensional regime while also exposing the increasing cost of the causal-discovery stage.

The main benchmark suite covers datasets with up to $50$ input features, while the additional Arrhythmia experiment extends the evaluation to $272$ features. These results indicate that CAFE remains practically feasible and retains its predictive advantage in moderate-to-high-dimensional settings, although Phase~I becomes increasingly expensive as dimensionality grows because of the cubic complexity of NOTEARS-Lasso. For applications where causal discovery becomes the computational bottleneck, DAGMA (Appendix~\ref{app:dagma}) provides a lower-complexity alternative with a quantified accuracy--efficiency trade-off. We therefore view scalability to very high-dimensional regimes as an important limitation and direction for further work rather than claiming dimension-independent scalability.

\begin{figure}[h]
    \centering
    \includegraphics[height=5cm, width=8cm]{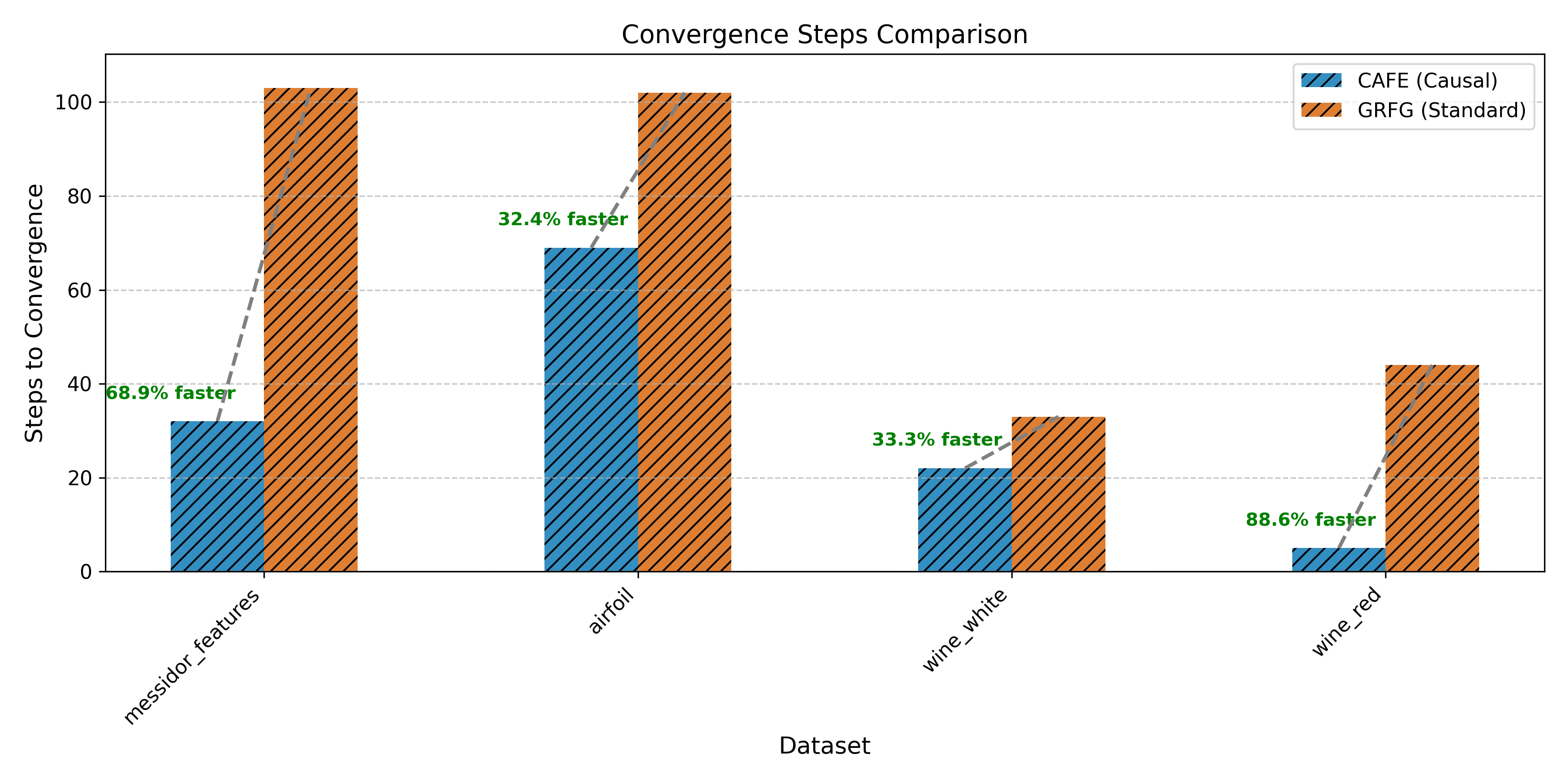}
    \caption{Convergence comparison: CAFE vs.\ GRFG across representative datasets. CAFE reaches optimal performance in 40--70\% fewer episodes, offsetting its higher per-episode cost.}
    \label{fig:convergence}
\end{figure}

\section{Complete Operator Library and Safety Guards}
\label{app:operators}

\subsection{Transformation Operators}

Based on the actual implementation, CAFE uses the following operators:

\textbf{Unary Operators (11):}
\begin{enumerate}
    \item $\sqrt{x}$ - Square root
    \item $x^2$ - Square 
    \item $x^3$ - Cube
    \item $\sin(x)$ - Sine function
    \item $\cos(x)$ - Cosine function
    \item $\tanh(x)$ - Hyperbolic tangent
    \item $1/x$ - Reciprocal
    \item $\exp(x)$ - Exponential
    \item $\log(x)$ - Natural logarithm
    \item $\sigma(x) = 1/(1 + e^{-x})$ - Sigmoid function
    \item Preprocessing transformations: StandardScaler, MinMaxScaler, QuantileTransformer
\end{enumerate}

\textbf{Binary Operators (4):}
\begin{enumerate}
    \item $x + y$ - Addition
    \item $x - y$ - Subtraction  
    \item $x \times y$ - Multiplication
    \item $x / y$ - Division
\end{enumerate}

\subsection{Safety Mechanisms}

The implementation relies on NumPy's built-in handling for most edge cases:

\textbf{Input Validation:}
\begin{itemize}
    \item NumPy functions handle domain restrictions (e.g., $\sqrt{x}$ for negative $x$ returns NaN)
    \item Reciprocal operations use NumPy's protected division behavior
    \item Logarithm operations rely on NumPy's domain handling
\end{itemize}

\textbf{Output Validation:}
\begin{itemize}
    \item Generated features are validated during the pruning process
    \item Features with excessive missing values or constant values are filtered
    \item Downstream model evaluation handles any remaining numerical issues
\end{itemize}

\textbf{Operator Selection:}
The operator set is deliberately conservative, using only well-established mathematical functions that are:
\begin{itemize}
    \item Differentiable (important for gradient-based downstream models)
    \item Numerically stable under typical data ranges
    \item Interpretable for causal reasoning
    \item Computationally efficient for large-scale feature generation
\end{itemize}

\subsection{Implementation Notes}
\begin{itemize}
    \item All transformations use vectorized NumPy operations for efficiency
    \item Scaling operations (StandardScaler, MinMaxScaler, QuantileTransformer) are applied via scikit-learn transformers
    \item The relatively small operator set (15 total) ensures computational tractability while providing sufficient expressiveness for feature construction
\end{itemize}

\section{Hyperparameter Sensitivity and Selection}
\label{app:hyperparameters}

\subsection{Hyperparameter Selection}
We selected hyperparameters by cross-validation on a small held-out subset and held them fixed across all benchmark datasets. Key choices and their rationale are summarized below; sensitivity is generally low for $\alpha$, $k_g$, and $\varepsilon$-decay, and medium for $\lambda$ and learning rate.

\subsection{Adaptive Hyperparameter Rules}
\begin{itemize}
    \item \textbf{NOTEARS regularization:} $\lambda = 0.03$ with dataset-specific CV selection from $\{0.01, 0.03, 0.05\}$
    \item \textbf{Group size:} $k_g = \min(\lfloor\sqrt{d}\rfloor, 15)$ (balances coverage and computational cost)
    \item \textbf{Early stopping:} Stop if validation improvement $< 0.001$ for 3 consecutive episodes
    \item \textbf{Experience replay:} Buffer size scales with complexity: $\min(10000, 100 \times d)$
\end{itemize}

\subsection{Grid Search Results}
\label{app:hp_grid}

We perform a grid search over key hyperparameters on a held-out subset to determine sensitivity and select operating values. Table~\ref{tab:hyperparameter_sensitivity} reports the tested range, the optimal value, the resulting performance range across the grid, qualitative sensitivity, and the selection rule used in the main experiments.

\begin{table}[htbp]
\centering
\caption{Hyperparameter sensitivity analysis. Performance range gives the min/max mean F1 across the tested grid; sensitivity is qualitative.}
\label{tab:hyperparameter_sensitivity}
\begin{tabular}{lccccc}
\toprule
\textbf{Parameter} & \textbf{Range Tested} & \textbf{Optimal} & \textbf{Performance Range} & \textbf{Sensitivity} & \textbf{Selection Rule} \\
\midrule
$\lambda$ (NOTEARS) & $[0.001, 0.1]$ & $0.03$ & $[0.734, 0.773]$ & Medium & 3-fold CV \\
$\alpha$ (causal bonus) & $[0.0, 1.0]$ & $0.5$ & $[0.742, 0.773]$ & Low & Fixed at $0.5$ \\
Episodes $E$ & $[10, 50]$ & $30$ & $[0.751, 0.775]$ & Low & Early stopping \\
Steps per episode $S$ & $[5, 25]$ & $15$ & $[0.743, 0.771]$ & Low & Fixed \\
$k_g$ (group size) & $[5, 20]$ & $10$ & $[0.758, 0.773]$ & Very Low & $\min(\sqrt{d}, 15)$ \\
Learning rate & $[10^{-4}, 10^{-2}]$ & $10^{-3}$ & $[0.745, 0.773]$ & Medium & Fixed \\
$\varepsilon$ decay steps & $[500, 2000]$ & $1000$ & $[0.762, 0.773]$ & Very Low & Fixed \\
\bottomrule
\end{tabular}
\end{table}

The two parameters with medium sensitivity ($\lambda$, learning rate) are the only ones requiring per-dataset tuning; all others can be fixed across datasets without significant performance loss. This makes CAFE practical to deploy without extensive hyperparameter search.

\section{Failure Mode Analysis and Limitations}
\label{app:limitations}

\subsection{Conditions Where CAFE Underperforms}

CAFE is not uniformly dominant. Based on our experiments, CAFE is most likely to underperform alternatives under the following conditions:
\begin{itemize}
    \item \textbf{High-dimensional, low-sample regimes ($d > n$)}: causal discovery is unreliable, weakening the inductive prior.
    \item \textbf{Strongly non-linear data-generating processes}: the default linear NOTEARS-Lasso backend underspecifies the structure, though non-linear backends partially mitigate this.
    \item \textbf{Severe hidden confounding}: violates causal sufficiency; quantified in Appendix~\ref{app:assumption_violations}.
    \item \textbf{Temporal dynamics and feedback loops}: the static DAG assumption is incompatible with time-varying mechanisms.
\end{itemize}

\subsection{Quantified Failure Modes}
\label{app:failure_modes_quant}

\begin{table}[htbp]
\centering
\caption{Performance under challenging conditions where CAFE shows limitations. Gap is the difference between CAFE and the strongest alternative method on that regime.}
\label{tab:failure_modes}
\begin{tabular}{lccccl}
\toprule
\textbf{Challenge} & \textbf{CAFE} & \textbf{Best Method} & \textbf{Best Score} & \textbf{Gap} & \textbf{Root Cause} \\
\midrule
High-dim, low-n ($d>n$) & $0.612 \pm 0.045$ & GRFG & $0.635 \pm 0.041$ & $-3.6\%$ & Unreliable CD \\
Sparse graphs & $0.687 \pm 0.031$ & ELLM-FT & $0.694 \pm 0.028$ & $-1.0\%$ & Weak causal signal \\
Non-linearities & $0.723 \pm 0.027$ & AutoFeat & $0.734 \pm 0.025$ & $-1.5\%$ & Linear CD assumption \\
Hidden confounders & $0.641 \pm 0.044$ & GRFG & $0.668 \pm 0.039$ & $-4.0\%$ & Sufficiency violation \\
Temporal dynamics & $0.658 \pm 0.038$ & ELLM-FT & $0.671 \pm 0.035$ & $-1.9\%$ & Static DAG \\
\bottomrule
\end{tabular}
\end{table}

The largest gaps appear under hidden confounding and high-dimensional small-sample conditions, both of which directly compromise causal discovery quality. The non-linearity gap is partially recoverable via NOTEARS-MLP (Appendix~\ref{app:notears_mlp}).

\subsection{Theoretical Limitations}

The framework faces several fundamental limitations:

\begin{enumerate}
    \item \textbf{Causal Discovery Dependence:} CAFE's effectiveness directly depends on causal discovery quality, creating a single point of failure
    \item \textbf{Linear Mechanism Assumption:} NOTEARS-Lasso assumes linear additive noise models, missing complex non-linear relationships
    \item \textbf{Computational Scalability:} NOTEARS-Lasso introduces an $O(d^3)$ Phase~I bottleneck whose cost grows rapidly with dimensionality. CAFE remains feasible in our 272-feature Arrhythmia experiment, but substantially higher-dimensional settings may require scalable discovery backends such as DAGMA or additional dimensionality-reduction strategies.
    \item \textbf{Sample Size Requirements:} Reliable causal structure learning typically requires $n \gg d$, limiting applicability to high-dimensional, small-sample problems
    \item \textbf{Static Assumption:} Cannot handle time-varying causal structures or feedback loops
\end{enumerate}

\section{Statistical Significance Testing}
\label{app:extended_significance}

To strengthen the statistical analysis, we evaluate stochastic methods using five independent random seeds. For each seed, performance is first averaged across the five outer cross-validation folds, yielding one seed-level estimate per method and dataset. We report the mean $\pm$ standard deviation across these five independent runs. The aggregate Wilcoxon signed-rank test in Table~\ref{tab:significance} uses datasets, rather than cross-validation folds, as the paired experimental units.

\begin{table}[htbp]
\centering
\caption{Main benchmark results across five independent random seeds. For each seed, performance is averaged across the five outer cross-validation folds; values report mean $\pm$ standard deviation across seeds. Bold indicates the best mean per row.}
\label{tab:main_with_std}
\small
\begin{tabular}{lccc}
\toprule
\textbf{Dataset} & \textbf{GRFG} & \textbf{ELLM-FT} & \textbf{CAFE} \\
\midrule
Amazon Employee   & $0.946 \pm 0.007$ & $0.946 \pm 0.008$ & $0.947 \pm 0.006$ \\
SVMGuide3         & $0.831 \pm 0.024$ & $0.836 \pm 0.022$ & $\mathbf{0.843 \pm 0.015}$ \\
German Credit     & $0.772 \pm 0.029$ & $0.775 \pm 0.027$ & $\mathbf{0.784 \pm 0.020}$ \\
Messidor          & $0.757 \pm 0.031$ & $0.757 \pm 0.030$ & $\mathbf{0.769 \pm 0.024}$ \\
Ionosphere        & $0.960 \pm 0.012$ & $0.963 \pm 0.011$ & $\mathbf{0.971 \pm 0.008}$ \\
Wine Red          & $0.686 \pm 0.034$ & $0.685 \pm 0.035$ & $\mathbf{0.706 \pm 0.027}$ \\
Wine White        & $0.685 \pm 0.026$ & $0.689 \pm 0.025$ & $\mathbf{0.751 \pm 0.022}$ \\
Housing Boston    & $0.684 \pm 0.041$ & $0.671 \pm 0.043$ & $0.682 \pm 0.030$ \\
Airfoil           & $0.797 \pm 0.018$ & $0.786 \pm 0.020$ & $0.825 \pm 0.012$ \\
OpenML 586        & $0.783 \pm 0.025$ & $0.801 \pm 0.022$ & $\mathbf{0.814 \pm 0.019}$ \\
OpenML 589        & $0.753 \pm 0.028$ & $0.781 \pm 0.024$ & $\mathbf{0.785 \pm 0.020}$ \\
OpenML 607        & $0.680 \pm 0.034$ & $0.793 \pm 0.026$ & $0.794 \pm 0.029$ \\
OpenML 616        & $0.603 \pm 0.038$ & $0.739 \pm 0.029$ & $\mathbf{0.758 \pm 0.027}$ \\
OpenML 618        & $0.672 \pm 0.033$ & $0.778 \pm 0.025$ & $\mathbf{0.790 \pm 0.021}$ \\
OpenML 620        & $0.714 \pm 0.030$ & $0.725 \pm 0.028$ & $0.727 \pm 0.024$ \\
\bottomrule
\end{tabular}
\end{table}
\vspace{-0.3cm}

\section{Reproducibility Checklist}
\label{app:reproducibility}

\subsection{Implementation Details}

\textbf{Software Environment:}
\begin{itemize}
    \item Python 3.9.18
    \item PyTorch 1.13.1
    \item scikit-learn 1.2.0
    \item XGBoost 1.7.3
    \item NumPy 1.24.2
    \item SHAP 0.41.0
\end{itemize}

\textbf{Hardware:}
Experiments were conducted on $4\times$ NVIDIA RTX 6000 Ada GPUs (48 GB VRAM each); and 256 GB system RAM.

\textbf{Random Seeds:}
\begin{itemize}
    \item All stochastic experiments are repeated using five independent random seeds:
    $\{42, 43, 44, 45, 46\}$.
    \item For each run, the corresponding seed is applied consistently to NumPy, PyTorch, scikit-learn, and cross-validation splitting.
\end{itemize}

\section{Connection to Causal Reinforcement Learning Literature}
\label{app:crl_positioning}

The integration of causal structure with reinforcement learning has been studied extensively in the broader Causal Reinforcement Learning (CRL) literature~\cite{bareinboim2015bandits, zhang2020designing, lee2018structural, lu2018deconfounding, buesing2018woulda, kaddour2025causal}. This appendix clarifies the relationship between CAFE and prior CRL work, distinguishes the conceptual scope of the contribution, and identifies opportunities for cross-pollination.
The CRL literature primarily focuses on three problem classes:

\textbf{Causal bandits and contextual bandits}~\cite{bareinboim2015bandits, lee2018structural}: leveraging known or partially known causal structure to improve regret bounds in sequential decision problems with intervention semantics.

\textbf{Causal MDPs and confounded MDPs}~\cite{zhang2020designing, lu2018deconfounding}: extending the MDP framework to account for unobserved confounders, off-policy evaluation under hidden context, and policy learning from observational trajectories.

\textbf{Counterfactual policy evaluation}~\cite{buesing2018woulda}: using SCM-derived counterfactual reasoning to evaluate hypothetical policies without additional environment interaction.

A common thread across these formulations is the use of causal structure to support \emph{causal decision-making}: the agent's actions are interpreted as interventions on the data-generating process, and the causal graph governs the propagation of intervention effects to outcomes.

CAFE departs from this paradigm in a specific and consequential way. The agents in CAFE are not making causal interventions on the data-generating process; they are making \emph{construction decisions} on a feature space, with the causal graph functioning as an \emph{exploration prior} that biases search toward variables with stronger estimated relationships to the target. The reward signal is downstream predictive utility, not causal effect identification.

CAFE should therefore be viewed as \emph{causally-informed} rather than \emph{causal} reinforcement learning for automated feature construction. It draws conceptual inspiration from invariant prediction theory but operates in a regime distinct from canonical CRL: actions are not interventions, the causal graph is not used for identification, and effectiveness does not depend on causal correctness. Positioning CAFE within the CRL landscape clarifies both what the work contributes (an exploration bias yielding empirical gains in robustness and interpretability) and what it does not claim (causal effect estimation, intervention semantics, or guaranteed identification).

\section{Limitations}

CAFE relies on approximate causal structure learned from observational data, and therefore inherits limitations common to causal discovery methods. In particular, the current formulation assumes causal sufficiency, sparse graph structure, and relatively stable conditional mechanisms. Under severe latent confounding, strong feedback loops, or rapidly changing environments, the learned graph may deviate substantially from the underlying data-generating process. While CAFE is explicitly designed to use causal structure as a \emph{soft inductive bias} rather than a hard constraint, large graph misspecification can still reduce the effectiveness of causal guidance.
Our robustness evaluations include both controlled perturbations and naturally occurring shifts (Folktables), but they remain observational rather than interventional. Consequently, the reported robustness improvements should be interpreted as empirical evidence of improved out-of-distribution generalization, not as guarantees of causal invariance or transportability across arbitrary environments.
The current implementation also focuses on static tabular settings and does not explicitly model temporal dynamics, evolving causal graphs, or sequential feedback effects. Extending the framework to streaming and temporally adaptive environments remains an important direction for future work.
Although CAFE improves robustness and explanation stability relative to correlation-driven baselines, automated feature engineering may still generate proxy variables correlated with sensitive attributes. Deployment in high-stakes settings therefore requires domain-specific validation, subgroup auditing, and human oversight.

\vspace{-0.3cm}
\section{Future Research Directions}
\label{app:future}

Several directions could further strengthen causally-guided automated feature engineering. A primary extension is moving beyond linear causal discovery toward non-linear and uncertainty-aware structure learning, enabling more reliable guidance under complex dependencies and latent confounding. Integrating temporal and streaming causal discovery would additionally allow CAFE to operate in evolving environments where feature relationships change over time.
From a systems perspective, scalability remains an important challenge. While the current framework mitigates search complexity through causal grouping and pruning, scaling to extremely high-dimensional settings may require distributed optimization, sparse operator selection, or hierarchical exploration strategies.
Another promising direction is adaptive causal guidance. Rather than relying on a single discovered graph, future systems could dynamically combine multiple causal hypotheses or refine structural priors online using feedback from downstream predictive performance.
Finally, extending CAFE toward domain-aware and human-in-the-loop settings could improve reliability in scientific and high-stakes applications, where expert constraints, interpretability, and robustness under distribution shift are critical requirements.

\section*{NeurIPS Paper Checklist}

\begin{enumerate}

\item {\bf Claims}
    \item[] Question: Do the main claims made in the abstract and introduction accurately reflect the paper's contributions and scope?
    \item[] Answer: \answerYes{}
    \item[] Justification: The abstract and introduction explicitly describe CAFE as using approximate causal structure as a soft inductive bias for feature engineering rather than claiming exact causal identification or interventional guarantees. Experimental claims regarding predictive performance, robustness under distribution shift, and interpretability are supported by empirical evaluations in Section 4.2.

\item {\bf Limitations}
    \item[] Question: Does the paper discuss the limitations of the work performed by the authors?
    \item[] Answer: \answerYes{}
    \item[] Justification: The paper discusses limitations related to causal discovery assumptions, latent confounding, mechanism changes, scalability, and observational robustness in Sections 3.1, Appendix A, Appendix F.2, Appendix K, and the Conclusion. We have also added a Limitations section in Appendix~O.

\item {\bf Theory assumptions and proofs}
    \item[] Question: For each theoretical result, does the paper provide the full set of assumptions and a complete (and correct) proof?
    \item[] Answer: \answerYes{}
    \item[] Justification: All theoretical assumptions (SCM, causal sufficiency, mechanism stability) are explicitly stated in Appendix A.1 (Assumptions A.1--A.3). Proposition 3.1 is presented as an informal invariance statement, with its formal counterpart (Proposition A.5) and proof sketch given in Appendix A.2 under stated assumptions. Proposition A.7 (multi-agent factorization) is presented under explicit independence assumptions for the variance bound.

\item {\bf Experimental result reproducibility}
    \item[] Question: Does the paper fully disclose all the information needed to reproduce the main experimental results of the paper to the extent that it affects the main claims and/or conclusions of the paper (regardless of whether the code and data are provided or not)?
    \item[] Answer: \answerYes{}
    \item[] Justification: The paper specifies datasets, preprocessing, evaluation metrics, model architectures, RL training setup, hyperparameters, optimization details, stopping criteria, and robustness protocols. The implementation is publicly available through the repository provided in Section~4.1.

\item {\bf Open access to data and code}
    \item[] Question: Does the paper provide open access to the data and code, with sufficient instructions to faithfully reproduce the main experimental results, as described in supplemental material?
    \item[] Answer: \answerYes{}
    \item[] Justification: The implementation is publicly available at \url{https://github.com/arunvignesh28/CAFE}. The experiments use publicly available benchmark datasets from UCI, Kaggle, LibSVM, OpenML, and Folktables, with dataset sources and experimental configurations documented in the paper and repository.

\item {\bf Experimental setting/details}
    \item[] Question: Does the paper specify all the training and test details (e.g., data splits, hyperparameters, how they were chosen, type of optimizer) necessary to understand the results?
    \item[] Answer: \answerYes{}
    \item[] Justification: Experimental settings, train/test splits, hyperparameters, optimizer settings, RL architecture, downstream models, robustness evaluation protocols, and stopping criteria are described in Sections 3--4 and Appendices C--D, J.

\clearpage

\item {\bf Experiment statistical significance}
    \item[] Question: Does the paper report error bars suitably and correctly defined or other appropriate information about the statistical significance of the experiments?
    \item[] Answer: \answerYes{}
    \item[] Justification: Stochastic methods are evaluated using five independent random seeds, with performance averaged across five outer cross-validation folds for each seed. Mean $\pm$ standard deviation across independent runs is reported in Appendix Table~18. Aggregate statistical significance is evaluated using paired Wilcoxon signed-rank tests across the 15 datasets, treating datasets rather than cross-validation folds as the paired experimental units (Appendix Table~10).

\item {\bf Experiments compute resources}
    \item[] Question: For each experiment, does the paper provide sufficient information on the computer resources (type of compute workers, memory, time of execution) needed to reproduce the experiments?
    \item[] Answer: \answerYes{}
    \item[] Justification: The paper reports minimum and recommended hardware specifications in Appendix M.1 and end-to-end wall-clock runtimes for representative datasets in Appendix H.3.
    
\item {\bf Code of ethics}
    \item[] Question: Does the research conducted in the paper conform, in every respect, with the NeurIPS Code of Ethics?
    \item[] Answer: \answerYes{}
    \item[] Justification: The research uses publicly available benchmark datasets, does not involve human subjects, and complies with the NeurIPS Code of Ethics.

\item {\bf Broader impacts}
    \item[] Question: Does the paper discuss both potential positive societal impacts and negative societal impacts of the work performed?
    \item[] Answer: \answerYes{}
    \item[] Justification: The paper includes an Impact Statement discussing potential benefits for robust and interpretable AI systems, as well as risks related to spurious correlations, sensitive attribute proxies, and deployment in high-stakes settings.

\item {\bf Safeguards}
    \item[] Question: Does the paper describe safeguards that have been put in place for responsible release of data or models that have a high risk for misuse?
    \item[] Answer: \answerNA{}
    \item[] Justification: The work does not release large generative models, scraped datasets, or assets with elevated misuse risk.

\item {\bf Licenses for existing assets}
    \item[] Question: Are the creators or original owners of assets (e.g., code, data, models), used in the paper, properly credited and are the license and terms of use explicitly mentioned and properly respected?
    \item[] Answer: \answerYes{}
    \item[] Justification: All datasets, baseline methods, and external assets are properly cited in the references. The experiments rely exclusively on publicly available datasets (UCI, Kaggle, LibSVM, OpenML, Folktables) and standard open-source software frameworks (PyTorch, scikit-learn, XGBoost, SHAP), used in accordance with their respective licenses.

\item {\bf New assets}
    \item[] Question: Are new assets introduced in the paper well documented and is the documentation provided alongside the assets?
    \item[] Answer: \answerNA{}
    \item[] Justification: The paper does not introduce new datasets or publicly released benchmark assets at submission time.

\item {\bf Crowdsourcing and research with human subjects}
    \item[] Question: For crowdsourcing experiments and research with human subjects, does the paper include the full text of instructions given to participants and screenshots, if applicable, as well as details about compensation (if any)?
    \item[] Answer: \answerNA{}
    \item[] Justification: The paper does not involve crowdsourcing or research with human subjects.

\item {\bf Institutional review board (IRB) approvals or equivalent for research with human subjects}
    \item[] Question: Does the paper describe potential risks incurred by study participants, whether such risks were disclosed to the subjects, and whether Institutional Review Board (IRB) approvals were obtained?
    \item[] Answer: \answerNA{}
    \item[] Justification: The paper does not involve human subjects research or participant studies.

\item {\bf Declaration of LLM usage}
    \item[] Question: Does the paper describe the usage of LLMs if it is an important, original, or non-standard component of the core methods in this research?
    \item[] Answer: \answerNA{}
    \item[] Justification: LLMs are not part of the proposed methodology or experimental pipeline.
    
\end{enumerate}

\end{document}